\documentclass{article}
\usepackage{iclr2027_conference,times}
\usepackage{amsmath,amssymb,amsthm,mathtools}
\usepackage{graphicx}
\usepackage{booktabs}
\usepackage{subcaption}
\usepackage{natbib}
\usepackage{xcolor}
\definecolor{revcolor}{RGB}{0,0,0}   
\DeclareRobustCommand{\rev}[1]{{\color{revcolor}#1}}

\usepackage[colorlinks=true,linkcolor=black,citecolor=blue,urlcolor=blue]{hyperref}
\usepackage{url}
\usepackage[most]{tcolorbox}
\usepackage{enumitem}
\usepackage[T1]{fontenc}

\newtheorem{theorem}{Theorem}
\newtheorem{proposition}[theorem]{Proposition}

\theoremstyle{definition}

\theoremstyle{remark}
\newtheorem{remark}{Remark}

\newcommand{\bS}{\beta}                       
\newcommand{\WS}{W_{S}}                       
\newcommand{\WT}{\widehat{W}_{T}}             
\newcommand{\WR}{W_{R}}                       
\newcommand{\What}{\widehat{W}}               
\newcommand{\ncls}{n}                         
\newcommand{\ndim}{m}                         
\newcommand{\Ltot}{L}                         
\newcommand{\Lr}{L_{r}}                       
\newcommand{\Lt}{L_{t}}                       
\newcommand{\gpar}{g_{\parallel}}

\newcommand{\Ppar}{P_{\parallel}}
\newcommand{\Pperp}{P_{\perp}}
\newcommand{\Sig}{\Sigma}
\newcommand{\Hess}{H}
\newcommand{\Tr}{\operatorname{Tr}}
\newcommand{\E}{\mathbb{E}}
\newcommand{\Real}{\mathbb{R}}
\newcommand{\lr}{\eta}                        
\newcommand{\bs}{B}                           
\newcommand{\data}{D}                         
\newcommand{\xc}{z}                           
\newcommand{\xcopt}{\xc^{\ast}}
\newcommand{\lropt}{\lr^{\ast}}
\newcommand{\ar}{\alpha_{r}}                  
\newcommand{\at}{\alpha_{t}}                  
\newcommand{\aD}{\alpha_{D}}                  
\newcommand{\expg}{\gamma}                    
\newcommand{\exps}{\sigma}                    
\newcommand{\expp}{p}                         
\newcommand{\expq}{q}                         
\newcommand{\pwr}{a}                          
\iclrfinalcopy

\title{Optimizer-dependent training dynamics converge to the same one-third optimal data scaling}

\author{%
Hyunseok Lee$^{1}$, \;
Mihir Basil$^{2}$\thanks{Work done during an internship at MIT.}, \;
Yizhou Liu$^{1,\dagger}$, \;
Jeff Gore$^{1,\dagger}$ \\
$^{1}$Department of Physics, Massachusetts Institute of Technology, Cambridge, MA \\
$^{2}$Department of Physics, University of Pennsylvania, Philadelphia, PA \\
$^{\dagger}$Co-corresponding authors:
\texttt{liuyz@mit.edu}, \texttt{gore@mit.edu}
}

\begin{document}
\maketitle
\lhead{Preprint}

\begin{abstract}
Neural scaling, in which loss falls as a power law with training, is central to large language models, and one recent proposal is that a $1/3$ exponent emerges from learning peaked distributions. That account describes SGD, but models in practice are trained with adaptive optimizers. \rev{We study this in a minimal model of a language-model output layer, a single softmax layer trained online to reproduce the hard labels of a fixed teacher, comparing SGD, Adam, Muon and four related optimizers.} \rev{Adaptive optimizers require separating two exponents that coincide for SGD, and that the $1/3$ account therefore does not distinguish:} how fast the loss falls with training steps along a single run, and how fast the \rev{optimal loss falls with dataset size $D$, with the learning rate and batch size tuned for each $D$}. We show that the first, a dynamic exponent, is optimizer-specific while the second, an optimal data exponent, converges to $1/3$ across optimizers. \rev{We decompose the loss into a radial part, set by the growth of the weight norm, and a tangential part, set by misalignment with the teacher,} each decaying as a power law with dynamic exponents $\alpha_{r}$ and $\alpha_{t}$. Under SGD, both are close to $1/3$, so the data exponent is also $1/3$ across different learning rates. Under Adam the two separate: $\alpha_{r} \simeq 0.48$ but $\alpha_{t} \simeq 0.08$. Since the total loss is minimized when these two parts are balanced, the optimal learning rate is optimizer-dependent: $D$-independent for SGD but falls with $D$ for Adam. Yet tuned to that optimum, the loss returns to $D^{-1/3}$ for both. A stochastic-dynamics analysis explains why: the optimizers can trade decay speed between the two \rev{parts}, but they all fall\rev{, to within a few percent,} on a single dynamic exponent relation, $2\alpha_{r}+ \alpha_{t} = 1$, which fixes the optimal data exponent at $1/3$. Across seven optimizers, including Muon, the measured exponents are consistent with this relation, and the optimal-loss envelopes agree with $D^{-1/3}$ across them. \rev{Learning-rate schedules with bounded peak rates, as in practice, can improve the coefficient but leave the one-third exponent unchanged.} The optimizer sets how fast a model learns per step; tuned optimally, it changes the prefactor but not the rate at which loss falls per sample.
\end{abstract}

\section{Introduction}
\label{sec:intro}

The training of today's large language models (LLMs) is driven by neural scaling
laws \citep{hestness2017deep,kaplan2020scaling,hoffmann2022chinchilla}: test
loss falls as a power law in training dataset size and in model size, so
training longer with more parameters keeps improving performance. These laws now
guide how pre-training budgets are set. Their origin, however, is still debated,
and without an explanation we cannot confidently predict future scaling or say
how to improve it.

We focus on the scaling with dataset size. Recent proposals show that the exponent emerges from non-linearity without power law present in the data
\citep{liu2026onethird,kuhn2026boundary}. When a softmax must learn a peaked
distribution, as it does in LLMs, the loss then falls as a power law in training
time with exponent $1/3$. Under online, single-epoch pre-training the number of
steps is proportional to the dataset size, so this reads as a $1/3$ data
scaling, close to the exponents measured from open-source models and from
Chinchilla \citep{hoffmann2022chinchilla}. The prior works
\citep{liu2026onethird,kuhn2026boundary}, however, treat gradient flow or plain
SGD, while LLMs are trained with adaptive optimizers \citep{kingma2015adam}.

Which optimizer is used is not just a detail of implementation. An optimizer rescales the gradient at every step and changes the noise the
update carries, so how fast the loss falls under the emergent neural scaling could
depend on the choice. The same may hold for where the optimum sits: Adam remains standard for pre-training, alternatives such
as Muon \citep{jordan2024muon} are now used in production runs, the two are reported to behave
differently, and the learning rate that works best scales differently with batch
size and with budget depending on which is chosen
\citep{bjorck2024scaling,li2025predictable,bergsma2025power,ren2026hypersphere}.
 We therefore ask:

\begin{tcolorbox}[colback=red!5, colframe=red!50, boxrule=0.5pt,
  left=6pt, right=6pt, top=4pt, bottom=4pt]
\textbf{Question:} Which parts of neural scaling are set by the optimizer, and which are not?
\end{tcolorbox}

We find that the training dynamics is optimizer-dependent while the optimal data
scaling is not. Where SGD follows a single $1/3$ power law in training time, no
single exponent describes the Adam loss curves. The reason is that the loss
carries two channels at once, a radial one set by softmax saturation and a
tangential one set by misalignment, and each decays with its own exponent, $\ar$
and $\at$. For SGD both come out close to $1/3$, the value the prior works
predict \citep{liu2026onethird,kuhn2026boundary}; for Adam neither does, with
$\ar$ near $1/2$ and $\at$ near zero. The learning rate that minimizes the loss
at a fixed sample budget is optimizer-dependent as well, scaling differently
with both batch size and budget. Yet across seven optimizers the two exponents
obey a sum rule, $2\ar + \at = 1$\rev{, to within a few percent}, which forces the tuned loss to
$\Ltot^{\ast}(\data) \sim \data^{-1/3}$ no matter where on the line an optimizer
sits. An optimizer therefore sets how fast a model learns per step, not how much
it can learn per sample.
\begin{tcolorbox}[colback=red!3, colframe=red!40!black, boxrule=0.6pt,
  title={Our contributions}, fonttitle=\normalsize,
  left=6pt, right=6pt, top=4pt, bottom=4pt]
\begin{itemize}[leftmargin=1.2em, itemsep=2pt, topsep=2pt, parsep=0pt]
  \item \textbf{Loss decomposition.} The loss splits into a radial part and a
  tangential part, each following its own power law in training time. The two
  exponents $\ar$ and $\at$ depend on the optimizer but not on the learning rate
  or batch size.
\item \textbf{One-third data optimum.} How the optimal hyperparameters scale
  depends on the optimizer, yet the loss at that optimum scales as $\data^{-1/3}$ for every optimizer we test\rev{, at a constant learning rate and under learning-rate schedules with a bounded peak rate}.
  \item \textbf{Dynamic Exponent Relation.} The two exponents fall on the sum rule $2\ar + \at = 1$\rev{, to within a few percent}, and this is what produces the one-third optimal data
  scaling.
\end{itemize}
\end{tcolorbox}

In short, SGD and Adam produce different training dynamics on the same task yet
reach the same one-third data scaling once the learning rate is tuned.
Section~\ref{sec:model} sets up the toy model. Section~\ref{sec:results} reports
the experiments together with the theoretical results that account for them,
with the full derivations left to the appendices. Related work follows in
Section~\ref{sec:related}, and Section~\ref{sec:discussion} closes with what the
findings imply.

\section{Toy model}
\label{sec:model}

We want the minimal toy model that captures the power-law scaling behavior of
LLMs. Following \citet{liu2026onethird} and \citet{kuhn2026boundary}, we use a
single-layer network with a softmax output and a cross-entropy loss, mimicking
the language modeling (LM) head.

We use a teacher--student setup in which both networks share the same
architecture and the student is trained to match the teacher's output. The
teacher is a fixed matrix $\WT \in \Real^{\ncls \times \ndim}$ and the student a
trainable $\WS \in \Real^{\ncls \times \ndim}$, each mapping an
$\ndim$-dimensional input to $\ncls$ logits. Inputs $x \in \Real^{\ndim}$, which
play the role of hidden states, are drawn i.i.d.\ standard normal and RMS-normalized, as hidden states are before the output head of a language model, and the teacher labels
each one by the one-hot encoding of $\mathrm{Softmax}(\WT x) \in \Real^{\ncls}$. Because a one-hot label records only
which logit is largest, the teacher's scale does not matter and we take $\WT$ to
have unit norm. The label is therefore maximally peaked: it is the
infinite-inverse-temperature limit of the teacher in \citet{liu2026onethird} and
the hard-label construction of \citet{kuhn2026boundary}, and it removes the
teacher inverse temperature as a hyperparameter that would otherwise have to be
swept.

The student maps the same input to $q(x) = \mathrm{Softmax}(\WS x)$, and we
train it online on the cross-entropy loss
$\Ltot = \langle \mathrm{CE}(p(x), q(x)) \rangle_{x}$, drawing a fresh batch at
every step. Two things then matter. The first is the weight norm
$\bS = \lVert \WS \rVert_{F}$. Because the inputs are isotropic with fixed norm, $\bS$ sets
the scale of the student's logits and therefore acts as an inverse temperature:
the larger $\bS$, the sharper the student's distribution. Since the teacher's
label is one-hot, the student can go on reducing its loss only by growing $\bS$.
The second is the optimizer. We compare SGD, which follows the mini-batch
gradient directly, with Adam, which rescales each coordinate by its accumulated
second moment. That rescaling changes both the drift and the noise in the
update, and its net effect on the late-time dynamics is not something a single
update reveals---so we treat it as something to measure rather than to assume.

\begin{figure}[t]
\centering
\includegraphics[width=\linewidth]{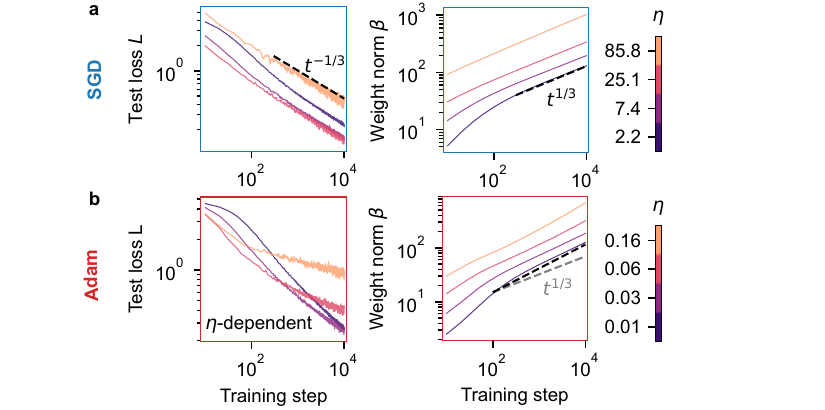}
\caption{\textbf{Optimizers produce qualitatively different dynamics on the same
task.} Online teacher--student softmax classification with $\ncls = 128$ classes
and input dimension $\ndim = 32$, trained with SGD (top row, batch size
$\bs = 256$) and Adam (bottom row, batch size $\bs = 64$). For each optimizer
$\lr$ and $\bs$ are chosen so that the tangential loss stays positive (Appendix~\ref{app:validity}); the exponents
do not depend on batch size for either optimizer (Figure~\ref{fig:fs2b}), so the
comparison is not affected by the choice. Color denotes learning rate $\lr$; the two rows use different ranges,
since the useful learning rates differ by roughly two decades between the
optimizers. \textbf{(a)} Training dynamics of SGD
follows $1/3$ time scaling. Left: test loss $\Ltot$ follows a single clean power
law across the full range, with fitted exponent $0.34 \pm 0.01$. Right: weight
norm $\bS = \lVert \WS \rVert_{F}$ is consistent with $\bS \sim t^{1/3}$ (fitted
$0.33 \pm 0.01$). \textbf{(b)} Training dynamics of Adam deviates from $1/3$
time scaling. Left: test loss curves bend over an extended transient and no
single exponent describes the trajectory. Right: weight norm expands faster than
$t^{1/3}$ throughout (fitted $0.44 \pm 0.04$).}
\label{fig:fig1}
\end{figure}

We next explore how the toy model trains under two different optimizers, SGD and
Adam. We fix $\ncls = 128$ and $\ndim = 32$, use a batch size of $256$ for SGD
and $64$ for Adam (chosen so that the tangential loss stays positive; the exponents do not depend on batch size for either
optimizer, Figure~\ref{fig:fs2b}), and show four learning rates spanning each optimizer's useful range,
$2.2$ to $85.8$ for SGD and \rev{$0.01$ to $0.16$} for Adam---the two optimizers
require learning rates some two decades apart on the same task
(Appendix~\ref{app:experiments}). For each run we
read off its time-scaling exponent, by which we mean the exponent of the power
law that the loss follows against training step at a fixed learning rate and
batch size---the quantity \citet{liu2026onethird} predict to be $1/3$. Under SGD
both the loss and the weight norm follow power laws whose exponents converge to
$1/3$ across the whole range of $\lr$, reproducing that prediction
(Figure~\ref{fig:fig1}a). Under Adam both quantities depart from it: the loss
decays far more slowly once $\lr$ grows past its optimum near \rev{$0.025$}, so no
single exponent describes the family of curves, and the weight norm, while still
a power law, grows faster than the $1/3$ that SGD shows
(Figure~\ref{fig:fig1}b). The time scaling is therefore optimizer-dependent, and
Adam does not follow the exponent that SGD does. What this does not settle is how
much a given number of samples buys once the learning rate is tuned, which is a
separate quantity with its own exponent; keeping the two apart is what the rest
of the paper is about.

\begin{tcolorbox}[colback=red!3, colframe=red!40!black, boxrule=0.6pt,
  title={Observation: Adam has no single exponent}, fonttitle=\normalsize,
  left=6pt, right=6pt, top=4pt, bottom=4pt]
  Under SGD the loss follows a single power law in training time; under Adam no single exponent describes it.
\end{tcolorbox}

\section{Results}
\label{sec:results}

\begin{figure}[t]
\centering
\includegraphics[width=\linewidth]{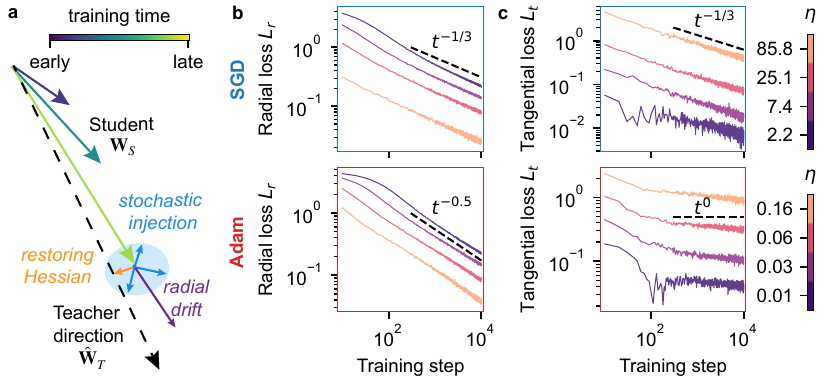}
\caption{\textbf{The loss splits into a radial and a tangential component with
distinct power-law dynamics and trade-offs between them.} \textbf{(a)} Geometry
of the decomposition. The
student $\WS$ is compared to the rescaled teacher $\WR = \bS \WT$, which carries
the student's current norm but the teacher's direction. The radial loss
$\Lr = \mathrm{CE}(p, p_{R})$ measures scale mismatch and the tangential loss
$\Lt = \Ltot - \Lr$ measures misalignment, which the restoring Hessian reduces
and stochastic injection from mini-batch noise sustains. \textbf{(b, c)} $\Lr$
and $\Lt$ against training step, for SGD (top) and Adam (bottom), colored by
learning rate over the same ranges as in Figure~\ref{fig:fig1}. Both decay as
power laws whose exponents are independent of the learning rate:
$\ar = 0.35 \pm 0.01$ and $\at = 0.33 \pm 0.02$ for SGD, $\ar = 0.48 \pm 0.03$ and $\at = 0.08 \pm 0.01$ for
Adam.}
\label{fig:fig2}
\end{figure}

In Adam, we observe loss dynamics that depend on the learning rate $\lr$. If the student were perfectly
aligned with the teacher direction, as assumed by the aligned student ansatz of
\citet{liu2026onethird}, the loss would track the growth of the weight norm, and its time exponent
would therefore be independent of $\lr$. We conclude that misalignment can
contribute non-negligibly to the loss. Indeed, while the restoring Hessian pulls
the student toward the teacher, stochastic injection from mini-batch sampling
noise kicks it away from perfect alignment (Figure~\ref{fig:fig2}a). This
motivates a loss decomposition built on an auxiliary rescaled teacher
$\WR = \bS \WT$, the teacher rescaled to the student's norm: we split the loss
into a radial part $\Lr = \mathrm{CE}(p, p_{R})$, with
$p_{R} = \mathrm{Softmax}(\WR x)$, and a tangential part $\Lt = \Ltot - \Lr$, so
that the radial loss measures the contribution of scale mismatch, that is an
unsaturated softmax, and the tangential loss measures that of misalignment. \rev{The same split appears, in different variables, in the boundary-layer analysis of \citet{kuhn2026boundary}: their test loss separates into the cost of an aligned student and a residual variance orthogonal to the teacher, and their dynamics give the $1/3$ exponent for SGD (Appendix~\ref{app:related-detail}).} With this decomposition we expect $\Lr$ to follow the time scaling of the weight norm
$\bS$, while $\Lt$ may have a time dependence of its own.

We first test the radial channel, where the decomposition predicts a loss set by
the weight norm. We track $\Lr$ against training step for both optimizers,
across the same range of learning rates as in Figure~\ref{fig:fig1}. Both SGD
and Adam show a clear power-law decay of $\Lr$ that is independent of the
learning rate (Figure~\ref{fig:fig2}b). For SGD the time-scaling exponent $\ar$
is close to $1/3$ ($0.35 \pm 0.01$), while for Adam it is faster, close to $1/2$
($0.48 \pm 0.03$). Both exponents are consistent with $\Lr \sim 1/\bS$, and plotting $\Lr$
directly against $\bS$ collapses the curves onto that law
(Figure~\ref{fig:fs2a}).
The radial channel therefore obeys the same law under both optimizers, the
$\Lr \sim 1/\bS$ predicted by \citet{liu2026onethird} and
\citet{kuhn2026boundary}; what differs between SGD and Adam is only how fast the
norm grows, and the two radial exponents follow from that.

We next turn to the tangential channel, the part of the loss that the aligned
student ansatz of \citet{liu2026onethird} assumes away. Measured the same way,
both SGD and Adam again give straight lines in a log-log plot, independent of
the learning rate (Figure~\ref{fig:fig2}c). For SGD the time-scaling exponent
$\at$ is again close to $1/3$ ($0.33 \pm 0.02$), but for Adam it is much slower, close
to zero ($0.08 \pm 0.01$). Each channel is therefore a clean power law with a
learning-rate-independent, optimizer-specific exponent, which explains what we
saw in Figure~\ref{fig:fig1}: the total loss is the sum of the two, and while
for SGD the exponents nearly coincide so the sum is again a single power law,
for Adam they are far apart, so the sum crosses over between them at a point
that depends on the learning rate.

\begin{tcolorbox}[colback=red!3, colframe=red!40!black, boxrule=0.6pt,
  title={Concept: Two channels}, fonttitle=\normalsize,
  left=6pt, right=6pt, top=4pt, bottom=4pt]
  The loss splits into a radial part set by the weight norm and a tangential part set by misalignment with the teacher. Each decays as its own power law, with optimizer-dependent yet hyperparameter-independent exponents $\ar$ and $\at$.
\end{tcolorbox}

Adam trades radial acceleration ($\ar > 1/3$) against tangential slowdown
($\at < 1/3$), which is what makes the total loss decay faster than $1/3$ early
and slow down later. And the experiment shows yet another trade-off, in the
coefficients rather than the exponents: in both SGD and Adam, raising the
learning rate lowers the radial loss $\Lr$ (Figure~\ref{fig:fig2}b) at the expense of a larger tangential loss 
$\Lt$ (Figure~\ref{fig:fig2}c), and batch size trades the same two quantities in the
opposite direction at a fixed sample budget (Figure~\ref{fig:fs2b}). The coefficient trade-off suggests the
balance between the two components may set the optimal hyperparameters. The
exponent trade-off adds that this balance must shift as training proceeds, since
the number of samples seen grows linearly with training time. We therefore
expect the optimal hyperparameters to be independent of the data size for SGD,
where the two time-exponents coincide, but to depend on it for Adam, where they
do not.

\begin{figure}[t]
\centering
\includegraphics[width=\linewidth]{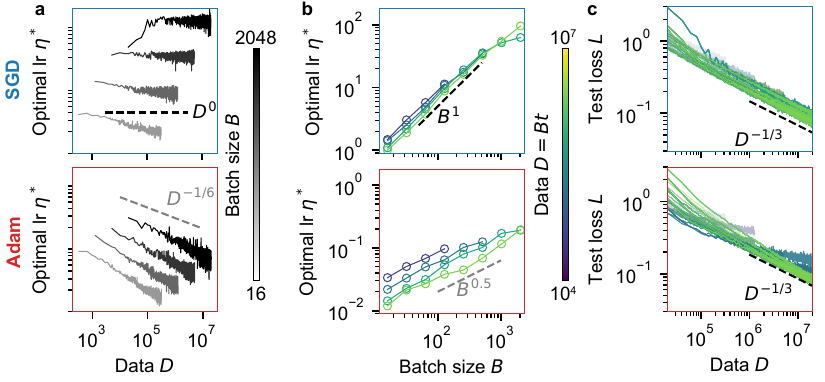}
\caption{\textbf{Optimal learning rate scales differently with $\bs$ and $\data$
for SGD and Adam, but the optimal loss follows one-third data scaling for both.}
Learning-rate sweeps over $8$ batch sizes $\bs \in [16, 2048]$ and $16$ learning
rates, spanning \rev{$[10^{-3}, 1]$} for Adam and $[10^{-1}, 10^{3}]$ for SGD, with
$\data = \bs t$ the number of online samples processed. In each panel the upper
row is SGD and the lower row Adam. Slopes $s_{B}$ and $s_{D}$ quoted in (a) and
(b) are from the joint log-log plane fit
$\log_{10} \lropt = s_{D}\log_{10} \data + s_{B}\log_{10} \bs + c$ described in Appendix~\ref{app:experiments}. \textbf{(a)} $\lropt$ against $\data$, one curve per
batch size (grayscale). SGD's optimum is nearly flat (\rev{$s_{D} = -0.013 \pm 0.004$}), while Adam's decreases (\rev{$s_{D} = -0.249 \pm 0.004$, shallower than the ``magic exponent'' $0.32$ reported for LLM pre-training} \citep{bjorck2024scaling}). \textbf{(b)} $\lropt$
against batch size, one curve per data size (color). SGD follows $\lropt \propto \bs$ (\rev{$s_{B} = 0.97 \pm 0.02$}), while Adam gives \rev{$s_{B} = 0.585 \pm 0.015$, close to the square-root rule of common practice}. \rev{Dashed lines in (a) and (b) are the idealized predictions $\lropt \propto \data^{-\expq/3}\bs^{1-\expp}$, with $\expp = \expq = 0$ for SGD and $\expp = \expq = 1/2$ for Adam.} \textbf{(c)} Test loss along the
optimal-$\lr$ envelope against $\data$, with individual runs colored by training
step and shaded by batch size. Both optimizers give the same optimal data scaling, $\Ltot^{\ast}(\data) \sim \data^{-1/3}$, with fitted exponents $\alpha_D = 0.341 \pm 0.001$ for SGD and $0.341 \pm 0.003$
for Adam. The quoted exponent is the pooled estimate across batch sizes of Appendix~\ref{app:experiments}.4. The $\lropt$ traces in (a) are visibly noisy at large $\data$ while the
envelope in (c) is clean; Figure~\ref{fig:fs3a} shows the underlying sweeps and the
flatness of the basin.}
\label{fig:fig3}
\end{figure}

We test this expectation by locating the optimal learning rate directly. We
sweep two parameters: $16$ learning rates, spanning $10^{-3}$ to \rev{$1$} for Adam
and $10^{-1}$ to $10^{3}$ for SGD, and $8$ batch sizes from $16$ to $2048$. For
each batch size and each step we take the minimum test loss over learning rates
and record the learning rate attaining it, then fit these locations in
$(\bs, \data)$ space to a log-log plane,
$\log_{10} \lropt = s_{D}\log_{10} \data + s_{B}\log_{10} \bs + c$ (Appendix~\ref{app:experiments}). \rev{The optimal learning rate turns out to be nearly data-independent for SGD ($s_{D} = -0.013 \pm 0.004$), while for Adam it decreases with $\data$ ($s_{D} = -0.249 \pm 0.004$, Figure~\ref{fig:fig3}a; a common-teacher sweep at fixed batch size gives $-0.227 \pm 0.007$, Appendix~\ref{app:corrections}). This is shallower than the ``magic exponent'' $0.32$ reported for the peak learning rate in LLM pre-training \citep{bjorck2024scaling}, a comparison we return to in Section~\ref{sec:discussion}.} Figure~\ref{fig:fig3}a therefore confirms what the
decomposed loss curves led us to expect: where the two time-exponents coincide,
as for SGD, the optimum is data-independent. Optimizers like Adam can break that
coincidence, which leads to a data-dependent optimum.

We next ask how the optimum depends on batch size. Because the coefficient
trade-off between radial and tangential losses is present for both SGD and Adam,
we expect a nonzero $s_{B}$ in both cases. \rev{Indeed, the plane fit gives $s_{B} = 0.97 \pm 0.02$ for SGD and $s_{B} = 0.585 \pm 0.015$ for Adam, close to the $\lropt \sim \sqrt{\bs}$ of standard practice; for Muon we find $0.570 \pm 0.004$, in agreement with the $0.558$ reported for Muon-based LLM training \citep{ren2026hypersphere}}
(Figure~\ref{fig:fig3}b). The curves taken at different $\data$ also collapse
for SGD while separating for Adam, mirroring Figure~\ref{fig:fig3}a. SGD and
Adam therefore differ in how the optimum moves with batch size as well as with
data size: every hyperparameter exponent we have measured separates the two
optimizers.

Now that we know $\lropt(\bs, \data)$, the natural next step is to ask what the
best achievable loss is at a given sample budget. We trace the optimal loss
envelope: for each $\data$, the lowest loss reached by any choice of
hyperparameters at that budget. This is a different object from a single
training curve, which is tuned for one budget and suboptimal at every other, and
it is the one that matters in practice. Given how differently the optimum
behaves under the two optimizers, we expect the envelopes
$\Ltot^{\ast}(\data)$ to differ as well. To our surprise, the envelope scales as
$\data^{-1/3}$ under both, with fitted exponents $\aD = 0.341 \pm 0.001$ for SGD and $0.341 \pm 0.003$
for Adam (Figure~\ref{fig:fig3}c). Tuning therefore washes out the differences we
measured in every individual exponent: the optimizer changes how training gets
there, but not what a given number of samples buys.

\begin{tcolorbox}[colback=red!3, colframe=red!40!black, boxrule=0.6pt,
  title={Result 1: One-third data optimum}, fonttitle=\normalsize,
  left=6pt, right=6pt, top=4pt, bottom=4pt]
Although the optimal hyperparameters scale differently with data size and batch size for SGD and Adam, the loss at
the optimum follows $\Ltot^{\ast}(\data) \sim \data^{-1/3}$ for both.
\end{tcolorbox}

We explain this convergence by modeling the learning dynamics as an
Ornstein--Uhlenbeck process. Although $\WS$ is high-dimensional, its dynamics
reduce to a radial and a tangential part, mirroring the loss decomposition, and
we solve for the steady state of each. The preconditioner enters the two
differently---in the radial channel it rescales the time derivative of the
weight norm, in the tangential channel the map from instantaneous loss to data
covariance---and that difference is what gives Adam two distinct time-exponents
where SGD, whose preconditioner is the identity, has one. For an optimizer whose
update $\sim -\lr g$ is scaled by $\bs^{\expp}\bS^{\expq}$, \rev{the hyperparameters
collapse into the single coordinate $\xc \equiv \lr\bs^{\expp-1}$, and the dynamics are unchanged by the rescaling $\bS \to \lambda\bS$, $\xc \to \lambda^{-\expq}\xc$, $\data \to \lambda^{3}\data$ (Appendix~\ref{app:corrected}). The two channels therefore take the form
\begin{equation}
\Lr = \frac{C}{\data^{1/3}\,b(u)}, \qquad \Lt = R(u)\,\Lr, \qquad u \equiv \xc\,\data^{\expq/3},
\label{eq:channels}
\end{equation}
where $\data^{1/3}b(u)$ is the weight norm reached, $R = \Lt/\Lr$ is the channel ratio, and $b$ and $R$ are functions whose shapes depend on the optimizer (Appendices~\ref{app:precond}, \ref{app:tangential} and~\ref{app:corrected}).} \rev{This form shows why the data exponent is shared: a larger step grows the norm faster, which sharpens the softmax and cuts the radial cost, but injects more misalignment, raising $R$. The total loss is $\Ltot = \data^{-1/3}\,C\,(1 + R(u))/b(u)$, so choosing the hyperparameters amounts to choosing $u$, whose optimum $u^{\ast}$ does not depend on $\data$: the optimal coordinate scales as $\xcopt \sim \data^{-\expq/3}$, and $\Ltot^{\ast} \propto \data^{-1/3}$ for any optimizer. The exponent $\expq$ enters only through how the coordinate combines with the budget. This holds including the feedback of tangential fluctuations on norm growth, which acts through $b$ and $R$ alone (Appendices~\ref{app:radial} and~\ref{app:corrected}).}

\begin{figure}[t]
\centering
\includegraphics[width=\linewidth]{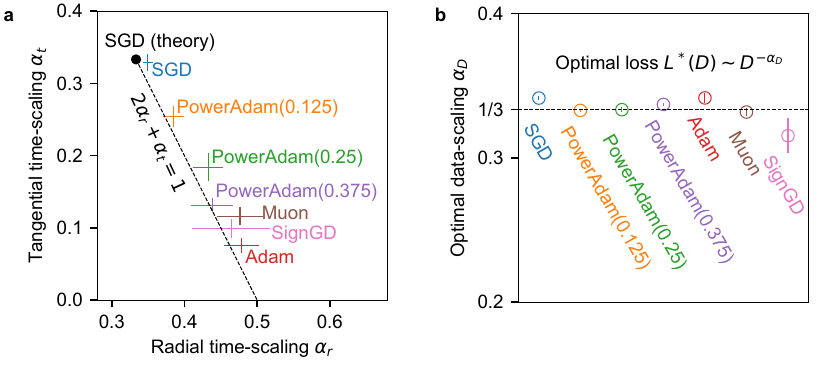}
\caption{\textbf{A sum rule for training exponents holds across optimizers, and
fixes the optimal data-scaling exponent at $1/3$.} \textbf{(a)} Measured radial
and tangential time-scaling exponents, $\Lr \sim t^{-\ar}$ and
$\Lt \sim t^{-\at}$, for SGD, the PowerAdam family $A = V^{-\pwr}$ at
$\pwr \in \{0.125, 0.25, 0.375, 0.5\ (\text{Adam})\}$, Muon, and SignGD. Crosses
give fit uncertainties. The dashed line is $2\ar + \at = 1$, a parameter-free
prediction with no fitted quantities; the black point marks the $\pwr = 0$
prediction $(1/3, 1/3)$. Adam sits at $\ar \simeq 0.48$,
above the nominal $2/5$, \rev{which with the tangential feedback of Appendix~\ref{app:corrected} corresponds to $\expq \approx 0.65$, consistent with its optimal-learning-rate exponent; late in training it approaches the nominal value (Appendix~\ref{app:late-time})}; Appendix~\ref{app:validity} rules out two alternative
explanations. \textbf{(b)} Optimal data-scaling exponent $\aD$, with
$\Ltot^{\ast}(\data) \sim \data^{-\aD}$. 
Six of the seven lie within 2.5\% of $1/3$, slightly above it by the finite-norm excess (Appendix~\ref{app:experiments}.4); \rev{SignGD reaches it only in the envelope over batch size and falls short at fixed batch} (Appendix~\ref{app:experiments}.4). This follows from (a): minimizing over $z$ gives $\alpha_D = (2\alpha_r + \alpha_t)/3$, which equals $1/3$ exactly on the line. Error bars are the spread across four lower fitting cutoffs (Appendix~\ref{app:experiments}).}
\label{fig:fig4}
\end{figure}

How do we know that this calculation is the right explanation for the $1/3$
scaling we measured? Usefully, it makes predictions beyond that scaling. Our
theory does not  derive an individual optimizer’s emergent $\expq$. \rev{Its parameter-free prediction concerns the two channels separately. In the decoupled approximation, which neglects the feedback between them, $b$ and $R$ are power laws, $b \propto u^{1/(3-\expq)}$ and $R \propto u^{3/(3-\expq)}$, and Equation~\eqref{eq:channels} reduces to a power law for each channel,
\begin{equation}
\Lr \sim (\xc\data)^{-1/(3-\expq)}, \qquad
\Lt \sim \xc^{\,2/(3-\expq)}\,\data^{-(1-\expq)/(3-\expq)},
\label{eq:channel-laws}
\end{equation}
so that $\ar = 1/(3-\expq)$ and $\at = (1-\expq)/(3-\expq)$, and the two exponents satisfy $2\ar + \at = 1$ whatever $\expq$ turns out to be (Appendix~\ref{app:corrected} gives the corrections).} To test this
sum rule we run five further optimizers: PowerAdam at three powers, whose
preconditioner scales as $V^{-\pwr}$ and for which Adam is the case
$\pwr = 1/2$; Muon; and SignGD \citep{bernstein2018signsgd}. Including SGD and Adam, the measured exponents of
all seven lie on the predicted line (Figure~\ref{fig:fig4}a). Muon and SignGD
are not members of the $V^{-\pwr}$ family---Muon's update is orthogonalized
rather than a diagonal rescaling---so their landing on the same line is evidence that the resulting scaling relations extend beyond the diagonal $V^{-a}$ family that motivates the calculation. Our calculation makes
two further predictions, neither displayed in the main figures and both free of parameters fitted to the quantities predicted: a relation between the dynamic exponent $\ar$ and the hyperparameter exponent $s_{D}$, and a fixed ratio between the two channels at the optimum, $\Lt/\Ltot = 1/3$. \rev{Both are exact in the decoupled approximation and shifted by the tangential feedback: the channel ratio becomes $0.30$ for SGD, against $0.296$ measured, and with the feedback included the estimates of $\expq$ from the dynamic exponents and from $s_{D}$ agree for Adam (Appendix~\ref{app:corrected}).} The evidence for the scaling description therefore does not rest on any single exponent: its dynamic, hyperparameter, envelope, and loss-ratio predictions are measured by separate procedures. The sum rule
therefore holds across optimizers whose individual exponents differ
substantially, which means the calculation captures the constraint linking the
two channels without needing to know what the preconditioner is.

Figure~\ref{fig:fig3}c established the $1/3$ optimum for two optimizers; the sum
rule implies it for every optimizer on the line, since
$\aD = (2\ar + \at)/3$ equals $1/3$ there. We repeat the envelope measurement for all seven. 
Although $\at$ varies by more than a factor of four across the family, the fitted $\alpha_D$ of six of them lie within $2.5\%$ of $1/3$, with a scatter across optimizers of $0.004$ (Figure~\ref{fig:fig4}b); \rev{the seventh, SignGD, reaches $1/3$ only in the envelope over batch size: at fixed batch its exponent falls short, increasingly so at small batch} (Appendix~\ref{app:experiments}.4).
Substituting the measured $\alpha_r$ and $\alpha_t$ into the same identity gives $0.34 \pm 0.01$ for SGD and $0.35 \pm 0.02$ for Adam, against fitted envelopes of $0.341$ for both; the two routes therefore agree within their uncertainties and both sit a few percent above $1/3$. Appendix~\ref{app:experiments} traces that common excess to the radial channel's approach to $L_r \propto \beta^{-1}$, which is measured directly and decreases with the weight norm.
The optimal exponent is therefore invariant while its two
components are not: an optimizer redistributes loss between the channels without
changing what optimally-tuned training buys per sample.

\begin{tcolorbox}[colback=red!3, colframe=red!40!black, boxrule=0.6pt,
  title={Result 2: Dynamic exponent relation}, fonttitle=\normalsize,
  left=6pt, right=6pt, top=4pt, bottom=4pt]
  The radial and tangential
time-exponents fall on a single relation, $2\ar + \at = 1$\rev{, to within a few percent}, across all seven optimizers, extending the one-third data optimum to all of them.
\end{tcolorbox}

\rev{\paragraph{Learning-rate schedules.} The results so far use a constant learning rate, whereas practical runs use schedules. We repeated the envelope measurement for SGD and Adam with cosine decay to $10\%$ of the peak, as in Chinchilla \citep{hoffmann2022chinchilla}, and to zero, and with warmup--stable--decay (WSD) \citep{hu2024minicpm}, tuning the peak rate at every budget within a window reaching two steps of our learning-rate grid, a factor of $2.5$, above the constant-rate optimum at that budget (Appendix~\ref{app:schedules}). Every schedule keeps $\aD \approx 1/3$ (Figure~\ref{fig:fig5}), within $0.03$ for both optimizers, and changes only the coefficient, lowering the loss by $30$--$60\%$ for cosine and WSD. With a wider window, schedules that anneal to zero reach $\aD \approx 1/3$ only at larger budgets, and the envelope can be steeper than $1/3$ only if the schedule anneals to zero and its peak rate is free to grow without bound, to orders of magnitude above the constant-rate optimum, which stability limits prevent in large networks (Appendix~\ref{app:schedules}).}

\begin{figure}[t]
\centering
\includegraphics[width=\linewidth]{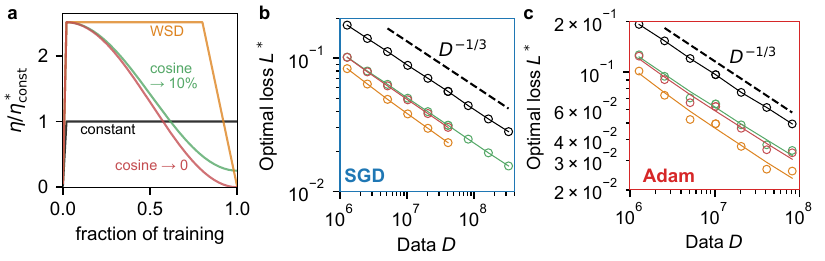}
\caption{\rev{\textbf{With a bounded peak rate, schedules change the coefficient but not the exponent.} \textbf{(a)} Learning rate against training progress for the four schedules (constant, cosine decay to $10\%$, WSD, and cosine decay to $0$), each at its optimal peak within the bound at the largest budget for SGD, in units of the constant-rate optimum; the peak is tuned at every budget. \textbf{(b, c)} Optimal loss against data $\data$ for SGD and Adam, with the peak rate tuned at every budget within $2.5$ times the constant-rate optimum at that budget. Points are the final losses of the best trained runs with peak rate at most two grid steps above the constant-rate optimum, located as in Figure~\ref{fig:fig3}c; lines are the same envelope with the bound at exactly $2.5$ times the constant-rate optimum, fitted across budgets, and the loss at the bound interpolated between grid rates. The dashed line is $\data^{-1/3}$. All envelopes are parallel to $\data^{-1/3}$, offset by the schedule (exponents in Table~\ref{tab:schedules}); the extended runs reach $\data \approx 3\times10^{8}$ (SGD, constant and cosine to $10\%$) and \rev{$8\times10^{7}$ (Adam, all four schedules)}. Five teachers per point. Appendix~\ref{app:schedules} gives other bounds and the unbounded case.}}
\label{fig:fig5}
\end{figure}

\section{Related work}
\label{sec:related}

Why adaptive methods behave differently from stochastic gradient descent has been studied both through convergence analyses and through empirical comparisons of their performance \citep{duchi2011adaptive, kingma2015adam, reddi2019convergence, wilson2017marginal, zhang2020adaptive, kunstner2023noise, kunstner2024heavy}. Our setting has no finite optimum to approach, so the optimizer's influence appears instead in how fast the loss falls.

A phenomenological picture of pre-training describes the loss landscape as a river valley, in which training progresses slowly along a flat direction while fluctuating across sharp ones \citep{cohen2021edge, wen2025understanding, cohen2024understanding, liu2025neural}. Many recent optimizers connect to this picture through their designs, which estimate curvature, precondition across matrix structure, or normalize the size of updates \citep{gupta2018shampoo, chen2023symbolic, liu2024sophia, vyas2025soap, jordan2024muon, yuan2024mars, liu2025focus}. A flat direction of this kind is familiar from separable data, where the cross-entropy loss keeps falling as the weight norm grows without bound, and a large literature characterizes the direction that gradient methods converge to as the norm grows \citep{soudry2018implicit, ji2018risk, nacson2019stochastic, lyu2019gradient}, and how that direction depends on the choice of optimizer \citep{gunasekar2018characterizing, wang2021implicit, wang2022does, zhang2024implicit, tsilivis2026flavors, fan2026implicit}. Our model offers a minimal, analytically tractable instance of such a landscape due to nonlinearity, where learning a peaked distribution whose samples come arbitrarily close to the decision boundary gives rise to the river.

The loss of language models falls as a power law in data and model size when hyperparameters are tuned \citep{hestness2017deep, kaplan2020scaling, hoffmann2022chinchilla, besiroglu2024chinchilla}, and theoretical accounts trace these exponents to structure in the data or to the strong non-linearity \citep{sharma2022manifold, bahri2024explaining, maloney2022solvable, michaud2023quantization, bordelon2024dynamical, paquette20244+, bordelon2025feature,liu2026superposition,liu2026inverse, liu2026onethird,kuhn2026boundary}. Related literature studies how the optimal hyperparameters scale with batch size, model width, and training budget \citep{goyal2017accurate, mccandlish2018empirical, smith2018dont, shallue2019measuring, yang2021tuning, malladi2022sdes, bjorck2024scaling, li2025predictable, bergsma2025power, ren2026hypersphere}. We bring these together in a model where both the training dynamics and the optimal hyperparameters depend on the optimizer, while the data exponent at the optimum does not.

\section{Discussion}
\label{sec:discussion}

From a single-layer toy model, we find that the loss separates into a radial channel and
a tangential channel, each decaying as a power law whose exponent depends on the
optimizer but not on the learning rate or batch size. Across seven optimizers,
these exponents vary substantially, yet all fall on the single relation $2\ar + \at = 1$\rev{, to within a few percent}, which
fixes the tuned loss at $\data^{-1/3}$ for every one of them. \rev{The evidence comes from separate procedures: for all seven optimizers, the channel exponents fall on the sum rule, the optimal-loss envelopes scale as $\data^{-1/3}$, and the channel ratio collapses onto a single function of the optimizer coordinate; for SGD, the theory with tangential feedback predicts the optimal learning rate, the channel ratio at the optimum and the point where $G$ changes sign with no fitted parameter (Table~\ref{tab:sgd-predictions}); and for Adam, it brings four independent estimates of $\expq$ into agreement (Table~\ref{tab:q-routes}).} Thus, how fast the loss decays, where the optimum sits, and what a sample budget buys are three faces of a single emergent scaling. An optimizer therefore changes how fast a
model learns per step, not how much it can learn per sample. We note that this concerns the
exponent, not the coefficient: optimizers differ in the prefactor of
$\Ltot^{\ast}(\data)$, which is what makes one preferable to another in
practice. The decomposition also yields a practical handle where the teacher direction is known: because the two channels sit in a fixed ratio at the optimum \rev{($1/3$ in the decoupled approximation, $0.30$ for SGD with the correction of Appendix~\ref{app:corrected})}, measuring $\Lt/\Ltot$ in a single run locates the optimal hyperparameters \rev{approximately} without a sweep
(Appendix~\ref{app:envelope}).

Our work has several limitations. The toy model is a single softmax layer, leaving multi-layer and Transformer architectures untested. \rev{Its inputs are isotropic; with power-law input spectra, the optimal envelope keeps $\aD = 1/3$, reaching it at larger budgets when the anisotropy is strong (Appendix~\ref{app:aniso}).} It also relies on
hard labels: with a finite-temperature teacher the student has a finite optimal
scale, so the power law would appear over an intermediate range, as \citet{liu2026onethird} report, rather than continuing indefinitely. \rev{Relatedly, the power law relies on the norm growing without bound: we use no weight decay, and how weight decay, which in language models acts mainly on weights followed by normalization layers, interacts with this mechanism is untested.} The two
emergent exponents separate as well, with Adam's batch scaling implying \rev{$\expp \simeq 0.42$ while its optimal-learning-rate exponent implies $\expq \simeq 0.68$}, where an idealized $A = V^{-\pwr}$ preconditioner would make them equal; every prediction we test depends on $\expq$ alone, so this does not affect our conclusions. \rev{The separation is pre-asymptotic: late in training, the PowerAdam family and Adam approach the idealized $\expp = \expq = \pwr$ (Appendix~\ref{app:late-time}); Muon's exponents, for which no such nominal value exists, change little, and predicting them from the optimizer's definition remains open.} Adam in particular is measured close to the edge of the regime the
analysis describes (Remark~\ref{rem:adam-edge}).  Finally, the decomposition itself needs
the teacher direction: $\Lr$, $\Lt$ and the tangent projector are all defined
relative to $\WT$, so for practical data without a ground-truth teacher, analysis will have additional conceptual complexity.
What does transfer is everything in Figures~\ref{fig:fig3}
and~\ref{fig:fig4}b---the scaling of $\lropt$ with batch size and with budget,
and the optimal exponent $\aD$---which is also everything for which an LLM
counterpart exists. \rev{The comparison is mixed: the tuned data exponent $1/3$ lies inside the range $0.28$--$0.37$ of the Chinchilla scaling laws \citep{hoffmann2022chinchilla, besiroglu2024chinchilla}, and Muon's batch exponent $s_{B} = 0.570 \pm 0.004$ agrees with the $0.558$ reported for Muon-based LLM training \citep{ren2026hypersphere}, but the budget exponents of the optimal learning rate, $s_{D} \approx -0.25$ for Adam and $-0.21$ for Muon, are shallower than the $0.32$ reported for LLM pre-training \citep{bjorck2024scaling, ren2026hypersphere}, and so is Adam's under the same cosine schedule ($-0.16$). Given the distance between a single softmax layer and a trained LLM, these comparisons} call for testing the theory at scale (Appendix~\ref{app:related-detail}).

Several of these limitations lead to one question: can the optimal data
scaling be made faster than $1/3$? One route is the local geometry of the data near decision boundaries. The
exponent follows from the loss and the gradient noise localizing near the
decision boundary with a regular margin density; if the density instead vanishes
there, so that fewer samples come close to the boundary, the sum rule shifts and the optimal
exponent accelerates past $1/3$, while the shifted values remain
optimizer-independent (Appendix~\ref{app:margin}). \rev{A second route is the learning-rate schedule. The one-third optimum survives schedules once the peak rate is bounded (Figure~\ref{fig:fig5}); only with an unbounded peak do schedules that anneal to zero give steeper envelopes (Appendix~\ref{app:schedules}). How the stability limit of a deep network bounds the peak, and how much annealing can gain within that bound, is open.}

Our results are reminiscent of universality classes in statistical mechanics.
There too, exponents measured in very different systems collapse onto shared
values, and the reason is that a diverging scale localizes the physics to a
singularity, so only near-boundary behavior survives. Our sum rule plays the role
of a scaling relation among critical exponents, such as the Rushbrooke relation~\citep{rushbrooke1963thermodynamics},
in constraining how the individual exponents must move together. Read this way,
the universality we find across optimizers does not bound how fast future models
can scale. It indicates where one has to change the problem in order to move that
bound.

\subsubsection*{Use of AI tools}

In this work, we used generative AI tools for several tasks whose disclosure is
required. They assisted in sharpening several steps of our mathematical derivations. They also assisted writing the proofs in Appendix from author-supplied derivations and provided feedback on methodology and experiments. 

We have not used generative AI tools for the conceptual framework of this work, the
proposal and refinement of the hypotheses tested here, and the design and production of the main
figures. Generating synthetic data sets, implementing methods, assisting with
translation, cleaning or reformatting data sets, and supporting qualitative or
thematic data analysis are either not applicable to this work or were carried
out by the authors.

Additionally, and for tasks whose disclosure is recommended, the manuscript text
was drafted with AI assistance: the authors supplied
paragraph-level notes fixing the content and order of each argument, an AI
assistant drafted the prose, and the authors revised. We
also used AI tools to identify relevant literature---including
\citet{mandt2017sgd}, whose trace identity we had derived independently before
locating it.

We have reviewed all AI-assisted work. We verified AI-assisted derivations, verified all numerical values from our own experiment are correctly cited, and read AI-identified references before citation. We take responsibility for the final content of this work, including text, claims and artifacts produced with the aid of generative AI.

\subsubsection*{Reproducibility statement}

All experiments use the single-layer teacher--student model described in
Section~\ref{sec:model} and Appendix~\ref{app:radial}. Appendix~\ref{app:experiments}
specifies the sweep ranges, the procedure for locating the optimal learning rate,
the plane fit, and the fitting windows and uncertainty estimates used for every
exponent reported in the main text\rev{; Appendix~\ref{app:common-teacher} describes the common-teacher sweeps used for Figure~\ref{fig:fig5} and Appendices~\ref{app:drift}, \ref{app:schedules} and~\ref{app:aniso}}. Code for the toy-model training runs and for all analysis and figures is provided in the supplementary material.

\bibliographystyle{iclr2027_conference}
\bibliography{references}

\appendix


\section{Relation to landscape and scaling-law literature}
\label{app:related-detail}
Two prominent explanations of power-law loss derive the exponent from power-law structure in the data~\citep{maloney2022solvable, bahri2024explaining, bordelon2024dynamical} and from the non-linearity of a softmax learning peaked distributions \citep{liu2026onethird,kuhn2026boundary} respectively; the mechanism here belongs to the second, and adds the optimizer's preconditioner to it. A separate literature asks instead how the optimal hyperparameters scale, either deriving rules from the invariance of a stochastic differential equation---the linear rule for SGD \citep{goyal2017accurate,smith2018dont} and the square-root rule for adaptive methods \citep{malladi2022sdes}---or measuring them empirically, from critical batch size \citep{mccandlish2018empirical,shallue2019measuring} to hyperparameter scaling laws for LLM pre-training \citep{bjorck2024scaling,li2025predictable,bergsma2025power,ren2026hypersphere}. The two have stayed apart: the first explains how fast the loss falls without reference to hyperparameters, the second locates the optimum without reference to the loss exponent. What we add is the relation between them, in which an optimizer's emergent scaling fixes the dynamic exponents, the hyperparameter exponents, and the data scaling at the optimum together.

In particular, \citet{kuhn2026boundary} derive coupled alignment and residual-variance dynamics from a boundary-layer analysis, and obtain learning rate schedules from them; \rev{their variables map onto ours exactly (Appendix~\ref{app:radial-eom}): our overlap $\omega$, squared norm $\bS^{2}$ and tangential displacement $\lVert e\rVert^{2}$ are their centered overlap $D$, centered norm $Q_{\mathrm{eff}}$ and residual variance $\Delta$, with $R = \Lt/\Lr = 6\Delta/\pi^{2}$. Their alignment equation, $\dot D \propto \lr(\pi^{2}/6 + \Delta)/D^{2}$, is our statement that the overlap is driven by the full loss; the drift and online-noise parts of their norm equation are the two contributions in Eq.~\eqref{eq:dbeta2}; and their online-noise integral $\mathcal{B}(\Delta)$ gives our noise factor, Eq.~\eqref{eq:gR}. For SGD at batch size one, up to the batch-size normalization of the noise term, the two sets of equations coincide; they reach them through statistical-physics order parameters, we through a stochastic process.}\label{app:corrected-kuhn} Their schedule result also separates observables in a way ours does not: annealing improves classification-error scaling while slowing cross-entropy decay in the regime they analyze. \rev{Their schedules are fixed functions of training time, applied to a single run; ours are defined relative to the budget, with the peak re-tuned at every budget, so a slower decay along one annealed run is compatible with the unchanged or steeper envelopes of Appendix~\ref{app:schedules}.} Their machinery is limited to SGD at batch size one, and they name richer optimizers as a future direction. Our machinery instead relates the tangential loss directly to the noise covariance, in the manner of \citet{mandt2017sgd} and adjacent stochastic-process treatments of training \citep{yaida2019fdr,malladi2022sdes}. This approach needs neither a covariance proportional to the Hessian (a condition reported to fail in practice \citep{zhang2026superlinear}) nor the two to commute, so extending it to state-dependent preconditioners (nonzero $\expq$) beyond SGD costs no further assumption about how noise relates to curvature. Combining our extensions of these two machineries connects state-dependent preconditioners to a boundary-localization law, from which power-law dynamics emerges without any power-law structure in the data.

\rev{Our specific numbers invite comparison with LLM pre-training, though we draw it cautiously. The Chinchilla scaling laws put the data exponent in the range $0.28$--$0.37$ \citep{hoffmann2022chinchilla,besiroglu2024chinchilla}, which brackets $1/3$; that scaling was measured with cosine schedules decaying the learning rate tenfold, the regime in which our one-third optimum holds (Appendix~\ref{app:schedules}). For the optimal learning rate, \citet{ren2026hypersphere} report a batch exponent $s_{B} = 0.558$ for Muon-based training in their HyperP framework, which agrees with our Muon value $0.570 \pm 0.004$. The budget exponents compare less well: the optimal learning rate falls as $\data^{-0.32}$ in \citet{bjorck2024scaling}, for the peak rate of Adam runs with cosine decay to $10\%$ of the peak, and in \citet{ren2026hypersphere}, whereas we measure $\data^{-0.23}$ to $\data^{-0.25}$ for Adam and $\data^{-0.21}$ for Muon at a constant rate, and $\data^{-0.16}$ for Adam with the same cosine schedule over our budgets. A single softmax layer need not reproduce these numbers, and we present neither the agreement nor the difference as a test of the theory.}

\section{Model, coordinates, and the radial equation of motion}
\label{app:radial}

\subsection{Model and coordinates}
Inputs are drawn i.i.d.\ from $\mathcal{N}(0, I_m)$ at every step and RMS-normalized, so that they lie uniformly on the sphere of radius $\sqrt{m}$ with $\mathbb{E}[xx^\top] = I_m$; the analysis below uses only this isotropy and the regular margin density of Appendix~\ref{app:radial}.2.
The teacher $\WT \in \Real^{\ncls \times \ndim}$ has unit
Frobenius norm and assigns the hard label
$y(x) = \arg\max_{c}\,(\WT x)_{c}$, so that the target is the one-hot vector
$p(x) = e_{y(x)}$. The student produces $q(x) = \mathrm{Softmax}(\WS x)$, with
per-sample loss $\ell(\WS; x) = -\log q_{y(x)}(x)$ and population loss
$\Ltot(\WS) = \E_{x}[\ell(\WS; x)]$.

Adding a constant to every logit leaves $q$ unchanged, so $\WS$ is defined only
up to a common shift of its rows. We remove this gauge freedom analytically by identifying $W_S$ with its canonical row-centered representative $P_cW_S$. All parameter-space directions and norms below refer to this identifiable subspace. The labels depend on the teacher only through a softmax and are therefore gauge-invariant, so we replace $\WT$ by its row-centered projection renormalized to unit norm, $P_c\WT/\|P_c\WT\|_F$, which leaves every label unchanged and places teacher and student in the same identifiable subspace; below, $\WT$ denotes this centered teacher.
In the experiments the rescaling is performed on logits by matching standard deviations across classes (Appendix~\ref{app:experiments}.1), which is unchanged by a common shift, so the experiments use the centered teacher automatically.

Write $\WS = \bS \What_{S}$ with $\lVert \What_{S} \rVert_{F} = 1$, and define
the radial and tangent projectors in vectorized parameter space,
\begin{equation}
\Ppar = \mathrm{vec}(\What_{S})\,\mathrm{vec}(\What_{S})^{\top}, \qquad
\Pperp = I - \Ppar .
\end{equation}
The radial gradient is $\gpar = \langle \What_{S}, \nabla \Ltot \rangle_{F}$.

For Appendix~\ref{app:tangential} we also need the raw tangent displacement $e$, obtained by resolving the student along the teacher direction,
\begin{equation}
\mathrm{vec}(\WS) = \omega\ \mathrm{vec}(\What_{T}) + e, \quad e^{\top}\mathrm{vec}(\What_{T}) = 0,\quad \bS^2 = \omega^2 + \lVert e\rVert^2.
\end{equation}
The coefficient of the teacher direction is the overlap $\omega$ rather than $\bS$, since $\bS$ is the norm of the whole student. Writing $\theta$ for the angle between $\WS$ and $\What_{T}$ gives $\omega = \bS \cos \theta$ and $\lVert e\rVert = \bS \sin \theta$, so $\omega = \bS\left(1 + O(\theta^2)\right)$ and the two agree at late times. To leading order in $\theta$, $e \simeq \beta\ \mathrm{vec}(\What_S - \What_T)$: the tangential part of $\What_S - \What_T$ is $O(\theta)$ while its radial part is $O(\theta^2)$ . Raw and angular displacement differ by a factor $\bS$, so their variances differ by $\bS^2$, a distinction that matters in Appendix~\ref{app:tangential} .

The same split applies to the loss. Define the rescaled teacher
$\WR = \bS \WT$, which carries the student's current norm but the teacher's
direction, and let $p_{R}(x) = \mathrm{Softmax}(\WR x)$. The radial loss
$\Lr = \mathrm{CE}(p, p_{R})$ is what a perfectly aligned student of norm $\bS$
still pays, and is therefore a function of $\bS$ alone; the tangential loss
$\Lt = \Ltot - \Lr$ is the extra cost of pointing the wrong way, and vanishes
when $\What_{S} = \WT$. Appendix~\ref{app:tangential} derives $\Lt$; the rest of
this section needs only $\Lr$.

Note that $\Lt$ is not positive by construction. The teacher direction minimizes
the loss only as $\bS \to \infty$; at finite $\bS$ the best direction at fixed
norm differs slightly from $\WT$, so a student that has found it pays
$\Lt < 0$. This finite-norm advantage decays faster than the stochastic
contribution that dominates at late times, so negative values are confined to an
early transient, visible at small batch size in Figure~\ref{fig:fs2b}.

\subsection{Boundary localization}

For an input with teacher label $y$, define the margin
\begin{equation}
\Delta(x) = (\WT x)_{y} - \max_{c \neq y}\,(\WT x)_{c} ,
\end{equation}
the gap between the largest teacher logit and the runner-up. It is non-negative
because the label is the argmax. We assume a \emph{regular decision boundary}:
the margin has a density $p_{\Delta}(\delta)$ on $\delta \geq 0$ with
$0 < p_{\Delta}(0) < \infty$, and the conditional moments of the per-sample
gradient vary continuously as $\delta \to 0^{+}$. Appendix~\ref{app:margin}
relaxes this.

At large $\bS$ the student's softmax saturates, so a sample with
$\bS \Delta \gg 1$ is classified with near-certainty and contributes
exponentially little to the loss or its gradient. The population loss and its
moments are therefore dominated by a boundary layer of width
$\Delta = O(\bS^{-1})$, the mechanism identified by \citet{kuhn2026boundary}. As
$\bS$ grows this layer thins, and the shrinking fraction of samples that still
contribute is what makes progress slow to a power law rather than an
exponential.

\subsection{The radial equation of motion}
\label{app:radial-eom}

Inside the boundary layer only the two largest logits matter, so a sample of
margin $\Delta$ contributes $\ell \simeq \log(1 + e^{-\bS\Delta})$. Evaluating at
the aligned configuration $\What_{S} = \WT$, where the loss reduces to $\Lr$ by
construction, and averaging over the margin distribution,
\begin{equation}
\Lr(\bS) = \int_{0}^{\infty} \log\!\left(1 + e^{-\bS\delta}\right)
p_{\Delta}(\delta)\, d\delta
= \frac{C_{L}}{\bS} + O(\bS^{-2}), \qquad
C_{L} = \frac{\pi^{2}}{12}\,p_{\Delta}(0),
\label{eq:Lrad}
\end{equation}
using $\int_{0}^{\infty}\log(1+e^{-u})\,du = \pi^{2}/12$
\citep{liu2026onethird}. Only the power is structural: a different boundary
kernel changes $C_{L}$ but not the exponent. Differentiating at fixed direction
gives the radial gradient
\begin{equation}
\gpar(\bS) = \frac{\partial \Lr}{\partial \bS}
= -C_{L}\bS^{-2} + O(\bS^{-3}).
\label{eq:gpar}
\end{equation}
\begingroup\color{revcolor}
This is the gradient at the aligned configuration. For a misaligned student two
further effects enter the radial dynamics at the same order.

\paragraph{Gradient at fixed direction.} At fixed direction $\What_{S}$, the radial loss scales as $\Lr \propto \bS^{-1}$
(Eq.~\eqref{eq:Lrad}), while the tangential loss scales as $\Lt \propto \bS$:
with $e \simeq \bS\,\theta$ and $\Hess \simeq \Hess_{0}/\bS$,
$\Lt \simeq \tfrac{1}{2}e^{\top}\Hess e = \tfrac{\bS}{2}\theta^{\top}\Hess_{0}\theta$.
The gradient along the student's own direction is therefore
\begin{equation}
\frac{\partial \Ltot}{\partial \bS}\bigg|_{\What_{S}} = \frac{\Lt - \Lr}{\bS} = \frac{G}{\bS},
\qquad G \equiv \langle \WS, \nabla \Ltot \rangle_{F} = \Ltot - \E_{x}[H(q)] ,
\label{eq:fixed-dir-grad}
\end{equation}
where the second form of $G$ is an exact per-sample identity for the softmax
cross-entropy, and $G \simeq \Lt - \Lr$ holds up to finite-norm corrections of order $1/\bS$. This gradient is not obtained by differentiating the stationary relation for $\Lt$, which describes how the tangential distribution re-equilibrates as $\bS$ changes rather than the landscape at fixed parameters.

\paragraph{Noise-driven growth.} Write the overlap with the teacher as $\omega = \langle \WS, \What_{T}\rangle_{F}$,
so that $\bS^{2} = \omega^{2} + \lVert e \rVert^{2}$. For SGD with rate $\lr$
and batch size $\bs$, one step changes the squared norm by
\begin{align}
\Delta\bS^{2} &= 2\langle \WS, \Delta\WS\rangle_{F} + \lVert \Delta\WS \rVert^{2}, \nonumber\\
\E[\Delta\bS^{2}] &= -2\lr G + \frac{\lr^{2}}{\bs}\Tr\Sig
= 2\lr(\Lr - \Lt) + 4\lr\Lt = 2\lr\Ltot ,
\label{eq:dbeta2}
\end{align}
using Proposition~\ref{prop:trace} with $A = I$ for the noise term and neglecting the squared mean gradient, which enters at $O(\lr^{2}\lVert\nabla\Ltot\rVert^{2})$. The drift
alone, $2\lr(\Lr - \Lt)$, does not drive the norm; the noise adds $4\lr\Lt$ and
the sum is the full loss. The overlap, being linear in $\WS$, receives no noise
contribution at the same order: its mean change is
$\lr\langle \What_{T}, -\nabla\Ltot \rangle_{F}$, and at leading order in the
boundary layer this force equals $\Ltot/\omega$, so that
$\E[\Delta\omega^{2}] \simeq 2\lr\Ltot$ as well. Late in training the
tangential distribution is stationary, $\lVert e\rVert^{2}$ is constant, and
$\bS$ and $\omega$ grow together. Appendix~\ref{app:drift} measures each term
of Eq.~\eqref{eq:dbeta2} directly.

\paragraph{The radial equation of motion.} Under SGD with dynamic time
$\tau = \lr t$, Eq.~\eqref{eq:dbeta2} gives $d\bS^{2}/d\tau = 2\Ltot
= 2C_{L}(1 + R)/\bS$ with $R \equiv \Lt/\Lr$, that is
\begin{equation}
\frac{d\bS}{d\tau} = C_{L}\bS^{-2}\left(1 + R\right).
\label{eq:beta-sgd-full}
\end{equation}
For SGD the ratio $R$ depends only on the optimizer coordinate
(Appendix~\ref{app:corrected}), so along a run at fixed rate the norm still grows
as $\tau^{1/3}$, with a prefactor raised by $(1 + R)^{1/3}$. Appendices~\ref{app:precond}--\ref{app:envelope} work in the decoupled approximation, which sets $R \to 0$ in this factor and neglects the dependence of the noise on misalignment; it gives
\par\endgroup
\begin{equation}
\frac{d\bS}{d\tau} = C_{L}\bS^{-2},
\qquad
\bS(\tau) = \left(3 C_{L} \tau\right)^{1/3} ,
\label{eq:beta-sgd}
\end{equation}
the $t^{1/3}$ growth measured in Figure~\ref{fig:fig1}a (fitted
$0.33 \pm 0.01$). \rev{Appendix~\ref{app:corrected} treats the factor $(1 + R)$ beyond the decoupled approximation.}

\section{Preconditioned radial dynamics}
\label{app:precond}

\subsection{The preconditioner as a scaling form}

Write the update as $\Delta \WS = -\lr A \hat{g}$, where $\hat{g}$ is the
mini-batch gradient and $A$ is the preconditioner the optimizer applies. What
matters for the late-time dynamics is not the detailed form of $A$ but how it
scales with the two quantities that vary slowly over training, the batch size
and the student norm:
\begin{equation}
A \sim \bs^{\expp}\,\bS^{\expq}.
\label{eq:Ascaling}
\end{equation}
The exponents $\expp$ and $\expq$ are emergent, in the sense that we measure them
from the macroscopic dynamics rather than read them off the update rule.

A concrete family makes Eq.~\eqref{eq:Ascaling} less abstract. For optimizers
that precondition by the diagonal second moment $V_{i} = \E[\hat{g}_{i}^{2}]$ of
the mini-batch gradient, write $A = V^{-\pwr}$. Boundary localization gives the
per-sample gradient covariance $\Sigma_{ii} \simeq s_{i}/\bS$, and since the
mini-batch mean has covariance $\Sigma/\bs$ while the square of the mean gradient
is subleading at late times, $V_{i} \simeq s_{i}/(\bs\bS)$. Hence
$A = V^{-\pwr} \sim (\bs\bS)^{\pwr}$, that is $\expp = \expq = \pwr$. SGD is the
case $\pwr = 0$ with $A = I$; the idealized Adam or RMSProp preconditioner is
$\pwr = 1/2$; the PowerAdam family interpolates.

Within this family the two exponents coincide, but the values we measure
separate them (Section~\ref{sec:discussion}), so $A = V^{-\pwr}$ is a special
case rather than the general form. Keeping $\expp$ and $\expq$ free costs nothing
in what follows: $\expp$ controls only the batch-size dependence and $\expq$ only
the scale dependence, and Appendix~\ref{app:envelope} shows that the sum rule and
the optimal exponent depend on $\expq$ alone.

\begin{remark}
We expect Eq.~\eqref{eq:Ascaling} to be generic. Any optimizer whose
preconditioner is determined by the current state, and whose only slowly varying
scale is $\bS$, should asymptotically take this form for some $\expp$ and
$\expq$; escaping it appears to require an anisotropy that does not scale
homogeneously with $\bS$. We do not prove this. The empirical support is that
Muon, whose update is orthogonalized and therefore not a diagonal rescaling,
nonetheless lies on the same line as the diagonal family
(Figure~\ref{fig:fig4}a).
\end{remark}

\subsection{The modified radial drift}

\begingroup\color{revcolor}
Appendix~\ref{app:radial} showed that the norm is driven by the full loss. With a
preconditioner of the scaling form Eq.~\eqref{eq:Ascaling}, the same steps give
\begin{equation}
\frac{d\bS}{d\tau} = K\,\bs^{\expp}\,\bS^{\expq-2}\left(1 + R\right),
\qquad R \equiv \frac{\Lt}{\Lr} ,
\label{eq:corrected-drift}
\end{equation}
so that tangential fluctuations feed back on radial growth through $R$; for
$\expq > 0$, $R \propto \xc\bS^{\expq}$ grows along a run. Below we integrate it in the decoupled approximation, $R \to 0$, from which the channel laws of
Appendices~\ref{app:tangential}--\ref{app:envelope} follow;
Appendix~\ref{app:corrected} treats the feedback.
\par\endgroup

With the preconditioner in place the radial drift of Appendix~\ref{app:radial}
becomes $d\bS/d\tau = -\langle \What_{S}, A \bar{g} \rangle_{F}$, where
$\bar{g} = \E_{x}[g]$. The radial gradient need not be aligned coordinate-wise
with a diagonal preconditioner, so we assume a \emph{stable preconditioned radial
shape}: there exists a finite $K > 0$, independent of $\bS$, $\bs$ and $\lr$,
such that
\begin{equation}
-\langle \What_{S}, A \bar{g} \rangle_{F}
= K\,\bs^{\expp}\bS^{\expq-2}\left[1 + o(1)\right].
\end{equation}
This is exact when $A$ becomes asymptotically scalar, and more generally requires
the normalized coordinate profiles of the radial gradient and of $A$ to approach
fixed shapes. Integrating $d\bS/d\tau = K\bs^{\expp}\bS^{\expq-2}$ gives
\begin{equation}
\bS^{3-\expq}(\tau) = \bS^{3-\expq}(0) + (3-\expq)\,K\,\bs^{\expp}\tau .
\end{equation}
Writing $\tau = \lr t$ and $\data = \bs t$ for the number of online samples, so
that $\tau = \lr \data/\bs$, the learning rate and batch size enter only through
the single combination
\begin{equation}
\xc \equiv \lr\,\bs^{\expp-1},
\label{eq:zdef}
\end{equation}
and at late times
\begin{equation}
\bS(\data) \simeq \left[(3-\expq)\,K\,\xc\,\data\right]^{1/(3-\expq)} .
\label{eq:beta-general}
\end{equation}
SGD ($\expp = \expq = 0$) recovers Eq.~\eqref{eq:beta-sgd} with $\xc = \lr/\bs$.
For the idealized Adam preconditioner, $\expq = 1/2$ predicts
$\bS \sim \data^{2/5}$; the measured exponent is $0.44 \pm 0.04$
(Figure~\ref{fig:fig1}b), above the SGD value and consistent with an effective
$\expq$ larger than the nominal $1/2$.

\section{The two loss channels}
\label{app:tangential}

\subsection{The radial channel}

Appendix~\ref{app:radial} gave the radial loss as a function of the norm,
$\Lr = C_{L}/\bS$, and Appendix~\ref{app:precond} gave the norm as a function of
the sample budget. Combining them,
\begin{equation}
\Lr(\data) \simeq C_{L}\left[(3-\expq)\,K\,\xc\,\data\right]^{-1/(3-\expq)}
\sim (\xc\data)^{-1/(3-\expq)} .
\label{eq:Lr-final}
\end{equation}
The radial channel therefore needs no stochastic input: it follows from a static
property of the loss surface together with the deterministic drift. The
tangential channel does not, and occupies the rest of this appendix.

\subsection{Tangential fluctuations and the stationary covariance}

Near alignment the tangential loss is a quadratic form in the raw tangent
displacement $e$ of Appendix~\ref{app:radial},
\begin{equation}
\Lt(e) = \tfrac{1}{2}\,e^{\top}\Hess e + O(\lVert e \rVert^{3}),
\end{equation}
where $\Hess$ is the Hessian restricted to the tangent subspace. Hence
\begin{equation}
\E[\Lt] = \tfrac{1}{2}\Tr(\Hess C), \qquad C = \E[e e^{\top}] ,
\label{eq:Lt-quadform}
\end{equation}
and the tangential loss becomes a property of the stationary distribution of
misalignment rather than something evaluated directly.

Because $\WR$ carries the student's full norm rather than its overlap, $\WS - \WR$ retains a radial component $\omega - \bS =- \lVert e \rVert^2/2\bS$, and expanding about $\WR$ produces a linear term $\gpar^*(\omega-\bS) = O(\lVert e \rVert^2/\bS^3)$ 
alongside the quadratic form. Since the tangential Hessian carries $\bS^{-1}$ against the radial gradient's $\bS^{-2}$, this term is smaller by $O(\bS^{-2})$ and is dropped. The tangential part of $\nabla L(\WR)$ does not vanish either --- the teacher direction is not the finite-norm minimizer --- and is the source of the $\Lt<0$ transient noted in Appendix~\ref{app:radial}.

Locally the displacement obeys a linear stochastic recursion,
\begin{equation}
e_{t+1} = e_{t} - \lr A\left(\Hess e_{t} + \xi_{t}\right), \qquad
\E[\xi_{t}] = 0, \quad \E[\xi_{t}\xi_{t}^{\top}] = \Sig/\bs ,
\label{eq:e-recursion}
\end{equation}
which contains exactly the two competing terms drawn in
Figure~\ref{fig:fig2}a: a restoring drift $-\lr A\Hess e$ set by the curvature,
and stochastic injection $-\lr A\xi$ set by mini-batch noise. Here $A$ may be any
symmetric positive-definite preconditioner that is approximately constant over
the local equilibration time. Propagating the covariance gives
\begin{equation}
C_{t+1} = (I - \lr A\Hess)\,C_{t}\,(I - \lr \Hess A)
+ \frac{\lr^{2}}{\bs}A\Sig A ,
\end{equation}
and imposing stationarity $C_{t+1} = C_{t} = C$, then cancelling $C$ and
dividing by $\lr$, yields the exact discrete preconditioned Lyapunov equation
\begin{equation}
A\Hess C + C\Hess A = \lr\,A\Hess C\Hess A + \frac{\lr}{\bs}A\Sig A .
\label{eq:lyapunov}
\end{equation}

\subsection{The trace identity}

Equation~\eqref{eq:lyapunov} could be solved for $C$, but by
Eq.~\eqref{eq:Lt-quadform} we only need $\Tr(\Hess C)$, and that follows from the
Lyapunov equation directly.

\begin{proposition}[Preconditioned tangential-loss identity]
\label{prop:trace}
At a stationary state of Eq.~\eqref{eq:e-recursion},
\begin{equation}
\E[\Lt] = \frac{\lr}{4\bs}\Tr(A\Sig)
+ \frac{\lr}{4}\Tr(\Hess C\Hess A)
= \frac{\lr}{4\bs}\Tr(A\Sig) + O(\lr^{2}) .
\label{eq:trace-identity}
\end{equation}
\end{proposition}

\begin{proof}
Left-multiply Eq.~\eqref{eq:lyapunov} by $A^{-1}$ and take the trace:
$\Tr(\Hess C) + \Tr(A^{-1}C\Hess A) = \lr\Tr(\Hess C\Hess A)
+ \tfrac{\lr}{\bs}\Tr(\Sig A)$. Cyclic invariance gives
$\Tr(A^{-1}C\Hess A) = \Tr(C\Hess) = \Tr(\Hess C)$, so
$2\Tr(\Hess C) = \lr\Tr(\Hess C\Hess A) + \tfrac{\lr}{\bs}\Tr(A\Sig)$.
Equation~\eqref{eq:Lt-quadform} then gives the result.
\end{proof}

The Hessian has cancelled from the leading term. Solving
Eq.~\eqref{eq:lyapunov} for $C$ would require knowing how the curvature and the
noise are oriented relative to one another; taking the trace does not. The
Hessian still shapes $C$, but it does not survive into the leading trace-level
tangential loss.

It is worth being precise about what Proposition~\ref{prop:trace} assumes. The
expansion is local and quadratic in $e$; the tangent distribution is taken to be
stationary; the step is small, so the $\Tr(\Hess C\Hess A)$ term is dropped; $A$
is approximately constant over the local equilibration time; and the tangent
distribution equilibrates faster than $\bS$, $A$, $\Hess$ and $\Sig$ drift. The
last two are separations of timescale rather than structural conditions, and
along a power-law trajectory the macroscopic quantities vary at relative rate
$O(1/\tau)$, so both separations widen as training proceeds. The small-step
condition behaves differently: the ratio of the dropped $\Tr(\Hess C\Hess A)$
term to the one retained scales as $\lr\bs^{\expp}\bS^{\expq-1}$, which decays
only for $\expq < 1$, and slowly when $\expq$ is close to one.

For Adam, Muon and SignGD we do not derive these conditions from the update rules. Adam's preconditioner fluctuates and is correlated with the current gradient; Muon's update is orthogonalized rather than a diagonal rescaling; sign-based updates discard gradient magnitude entirely. That these optimizers nonetheless obey the resulting scaling relations is an empirical finding, and the effective description $A \sim B^\expp \bS^\expq$ should be read as a fitted summary of their behavior rather than a consequence of their definitions.

What the identity does not require is any relation between the noise covariance
and the curvature: neither $\Sig \propto \Hess$ nor $[\Hess, \Sig] = 0$. This
matters because such a proportionality is a property of the landscape rather
than something one can arrange, and it is reported to fail in practice
\citep{zhang2026superlinear}.

\subsection{Scaling of the tangential channel}

Inserting the preconditioner scaling $A \sim \bs^{\expp}\bS^{\expq}$ of
Eq.~\eqref{eq:Ascaling} and the covariance $\Sig_{ii} \simeq s_{i}/\bS$ gives
$\Tr(A\Sig) = C_{T}\,\bs^{\expp}\bS^{\expq-1}[1 + o(1)]$ for a geometric
coefficient $C_{T}$, so that
\begin{equation}
\E[\Lt] \simeq \tfrac{1}{4}C_{T}\,\lr\,\bs^{\expp-1}\,\bS^{\expq-1}
= \tfrac{1}{4}C_{T}\,\xc\,\bS^{\expq-1} .
\end{equation}
Substituting the norm trajectory of Eq.~\eqref{eq:beta-general},
\begin{equation}
\Lt(\data) \sim \xc^{\,2/(3-\expq)}\,\data^{-(1-\expq)/(3-\expq)} ,
\label{eq:Lt-final}
\end{equation}
since $1 + (\expq-1)/(3-\expq) = 2/(3-\expq)$. Equations~\eqref{eq:Lr-final} and
\eqref{eq:Lt-final} are the two channel laws quoted in
Section~\ref{sec:results}.

For SGD ($\expp = \expq = 0$) the identity reproduces a result of
\citet{kuhn2026boundary} in a different form. Mode by mode,
$\E[e_{i}^{2}] = \lr \Sig_{ii}/(2\bs \Hess_{ii})$, and since boundary
localization makes both $\Hess_{ii}$ and $\Sig_{ii}$ scale as $\bS^{-1}$, their
ratio is asymptotically constant and $\E[e_{i}^{2}] \asymp \lr/\bs$. The raw
misalignment therefore saturates at a noise floor, exactly as they find at
$\bs = 1$. The tangential loss nevertheless decays, because
$\Lt = \tfrac{1}{2}\Tr(\Hess C)$ and the curvature flattens as
$\Hess \sim \bS^{-1}$ while the boundary layer thins. A constant misalignment
floor and a decaying tangential loss are thus the same statement, and the
measured $\at \simeq 1/3$ for SGD agrees with their result rather than
conflicting with it.

\section{The sum rule, the optimum, and the hyperparameter exponents}
\label{app:envelope}

\subsection{The two channels}

Equations~\eqref{eq:Lr-final} and \eqref{eq:Lt-final} give the two channels as
functions of the optimizer coordinate $\xc$ and the sample budget $\data$,
\begin{equation}
\Lr \simeq c_{r}\,(\xc\data)^{-\ar},
\qquad
\Lt \simeq c_{t}\,\xc^{\,2\ar}\,\data^{-\at},
\label{eq:channels-app}
\end{equation}
with coefficients $c_{r}, c_{t}$ that depend on the boundary geometry but not on
$\lr$, $\bs$ or $\data$, and with time-scaling exponents
\begin{equation}
\boxed{\;\ar = \frac{1}{3-\expq}, \qquad \at = \frac{1-\expq}{3-\expq}.\;}
\label{eq:exponents}
\end{equation}
We have used $1 - \at = 2/(3-\expq) = 2\ar$ to write the $\xc$-exponent of $\Lt$
as $2\ar$, which simplifies everything below. At fixed $\lr$ and $\bs$ we have
$\data \propto t$, so Eq.~\eqref{eq:exponents} gives the exponents of the decay
against training step plotted in Figure~\ref{fig:fig4}a.

\subsection{The sum rule}

Both exponents in Eq.~\eqref{eq:exponents} are fixed by the single emergent
quantity $\expq$, so eliminating it leaves a relation between them. From
$\ar = 1/(3-\expq)$ we have $\expq = 3 - 1/\ar$, and substituting,
\begin{equation}
\at = \frac{1-\expq}{3-\expq} = \left(\frac{1}{\ar} - 2\right)\ar
= 1 - 2\ar ,
\end{equation}
that is
\begin{equation}
\boxed{\;2\ar + \at = 1.\;}
\label{eq:sumrule}
\end{equation}
This is a statement about training dynamics at fixed hyperparameters; no
optimization enters it. Our theory does not predict $\expq$ for a given
optimizer, but whatever $\expq$ turns out to be, the pair $(\ar, \at)$ lies on
this line. That is the sense in which the dashed line of
Figure~\ref{fig:fig4}a is a parameter-free prediction.

\subsection{The optimum}

We now minimize the total loss $\Ltot = \Lr + \Lt$ over $\xc$ at fixed $\data$.

\begin{theorem}
\label{thm:envelope}
Under Eq.~\eqref{eq:channels-app}, the minimizing coordinate and the loss at that
minimum satisfy
\begin{equation}
\xcopt \propto \data^{-\expq/3},
\qquad
\Lr^{\ast} = 2\,\Lt^{\ast},
\qquad
\Ltot^{\ast} \propto \data^{-1/3},
\end{equation}
the last of these independent of $\expq$.
\end{theorem}

\begin{proof}
Differentiating Eq.~\eqref{eq:channels-app},
\begin{equation}
\frac{\partial \Ltot}{\partial \xc}
= -\ar\,c_{r}\,\xc^{-\ar-1}\data^{-\ar}
+ 2\ar\,c_{t}\,\xc^{\,2\ar-1}\data^{-\at} = 0 .
\end{equation}
Multiplying by $\xc/\ar$ gives
$c_{r}\xc^{-\ar}\data^{-\ar} = 2c_{t}\xc^{\,2\ar}\data^{-\at}$, which is exactly
$\Lr^{\ast} = 2\Lt^{\ast}$: at the optimum the two channels sit in a fixed
ratio, independent of $\expq$, of $\data$, and of the coefficients. Solving the
same equation for $\xc$,
\begin{equation}
\xc^{\,3\ar} = \frac{c_{r}}{2c_{t}}\,\data^{\at - \ar},
\end{equation}
and since $\at - \ar = -\expq/(3-\expq) = -\expq\,\ar$, we get
$\xcopt \propto \data^{-\expq/3}$. Substituting back,
\begin{equation}
\Ltot^{\ast} \propto (\xcopt\data)^{-\ar}
\propto \data^{-\ar(1 - \expq/3)}
= \data^{-\ar(3-\expq)/3}
= \data^{-1/3},
\end{equation}
using $\ar(3-\expq) = 1$. Both channels carry the same $\data$-exponent, so the
sum does too.
\end{proof}

The $\expq$-dependence cancels in the last step, which is the formal content of Result 2. The same minimization with $\ar$ and $\at$ treated as independently measured, rather than both expressed through $q$, gives $\aD = (2\ar + \at)/3$ — exactly one third of the sum-rule combination — so $\aD = 1/3$ if and only if $2\ar + \at = 1$. This is the form quoted in Figure~\ref{fig:fig4}b. Writing the $\xc$-exponent of $\Lt$ as $1-\at$ rather than $2\ar$, which is the same thing on the line but not off it, would instead give $\ar/(1+\ar-\at)$; at the measured exponents the two differ by about one percent.

\subsection{From dynamic to hyperparameter exponents}

The coordinate $\xc$ is not measured directly; the optimal learning rate is.
Two relations convert between them. At fixed batch size $\xc \propto \lr$, so
Theorem~\ref{thm:envelope} gives
\begin{equation}
\lropt \propto \data^{s_{D}}, \qquad s_{D} = -\expq/3 .
\end{equation}
Separately, because $\xcopt$ depends on $\data$ alone, the combination
$\lr\bs^{\expp-1}$ must be held fixed as $\bs$ varies, so
\begin{equation}
\lropt \propto \bs^{s_{B}}, \qquad s_{B} = 1 - \expp .
\end{equation}
The second follows from the definition of $\xc$ in Eq.~\eqref{eq:zdef} rather
than from the optimization.

Eliminating $\expq$ between $s_{D} = -\expq/3$ and $\ar = 1/(3-\expq)$ leaves a
relation between a dynamic exponent and a hyperparameter one,
\begin{equation}
\boxed{\;s_{D} = -1 + \frac{1}{3\ar}\;}
\label{eq:sD-ar}
\end{equation}
so the two are not independent. This curve passes through
$(\ar, s_{D}) = (1/3, 0)$, the SGD prediction, and terminates at
$(1/2, -1/3)$, which is the limit $\at \to 0$, or $\expq \to 1$, where the
tangential channel stops decaying altogether.
\rev{Equation~\eqref{eq:sD-ar} holds in the decoupled approximation. A measured $\ar$ includes two corrections it neglects: the finite-norm slope of $\Lr$ against $\bS$ (Remark~\ref{rem:adam-edge}) and, for $\expq > 0$, the tangential feedback of Appendix~\ref{app:corrected}. With both included, the estimates of $\expq$ from $\ar$ and from $s_{D}$ agree for Adam (Table~\ref{tab:q-routes}).}

Note that $s_{D}$ depends only on $\expq$ and $s_{B}$ only on $\expp$. The two
exponents of the preconditioner are therefore measured by separate
experiments, which is what allows the comparison discussed in
Section~\ref{sec:discussion}.

\begin{remark}[Multiple estimates of $q$ from Adam]
\label{rem:adam-edge}
Adam's three routes to the
exponent $\expq$ do not quite agree: the norm exponent of Figure~\ref{fig:fig1}b gives
$\expq \simeq 0.74$, the radial exponent $\ar = 0.48$ gives $0.92$, and \rev{$s_{D} = -0.249$ gives $0.75$ (Appendix~\ref{app:corrections}).} The first two are not independent measurements. The norm and radial-loss exponents are related by $\ar = \alpha_\beta \cdot |d \log \Lr / d \log \bS|$, which the asymptotic law $\Lr \propto 1/\bS$ fixes at $1$; measured over the fitting window this slope is $1.04-1.15$, decreasing monotonically with $\bS$ and taking nearly identical values under SGD ($1.078$) and Adam ($1.091$), as the $O(1/\bS)$ correction from expanding $p_\Delta$ about the boundary predicts. It accounts for the 9\% difference between the two exponents, so $\ar = 0.48$ is an accurate local slope while inverting it to $q = 0.92$ overstates the asymptotic exponent. \rev{With the tangential feedback of Appendix~\ref{app:corrected}, all routes agree at $\expq = 0.65$--$0.68$.} In the end, Adam carries the least discriminating weight of the seven optimizers. What establishes the line instead is the three PowerAdam points,
which span its interior, together with Muon and SignGD, which lie on it while
belonging to no $A = V^{-\pwr}$ family at all.
\end{remark}

\subsection{The loss ratio as a diagnostic}

The ratio of the two channels,
$R \equiv \Lt/\Lr = (c_{t}/c_{r})\,\xc^{\,3\ar}\,\data^{\ar-\at}$, equals
$1/2$ at the optimum by Theorem~\ref{thm:envelope}. Taking the ratio of $R$ to
that value eliminates the coefficients and the $\data$-dependence, leaving
$2R = (\xc/\xcopt)^{3\ar}$, that is
\begin{equation}
\xcopt = (2R)^{-1/(3\ar)}\,\xc .
\label{eq:diagnostic}
\end{equation}
A single run at any $(\lr, \bs)$ therefore locates the optimum
multiplicatively, without a sweep. Since $\expp < 1$, lowering $\xc$ means
raising $\bs$ or lowering $\lr$: if $R > 1/2$ the tangential channel is carrying
too much of the loss and one should raise the batch size or lower the learning
rate, and if $R < 1/2$ the reverse.

Equivalently, since $\Ltot = \Lr + \Lt = 3\Lt$ at the optimum,
\begin{equation}
\boxed{\;\Lt/\Ltot = 1/3 \quad \text{at the optimum \rev{(decoupled)}.}\;}
\end{equation}
The tangential channel carries one third of the loss when the hyperparameters are tuned\rev{, in the decoupled approximation; the feedback of Appendix~\ref{app:corrected} lowers this to $0.30$ for SGD, as measured}. Unlike the time-scaling and data-scaling exponents
this is a pure number, requiring no asymptotics in $\data$; testing it needs only the plane fit that locates $\xcopt$, and no parameter fitted to the ratio itself.

Figure~\ref{fig:fs4c} tests both statements at once. For every optimizer,
$\Lt/\Ltot$ measured across all learning rates, batch sizes and training times
collapses onto a single function of $\xc/\xcopt$, and that function passes
through $1/3$ at $\xc = \xcopt$. The collapse confirms that the hyperparameters
enter only through $\xc$; the crossing confirms the ratio. Deviations appear only for $\Lt/\Ltot \lesssim 0.1$, where $\Lt$ is a
small difference of two larger quantities and the finite-norm transient of
Appendix~\ref{app:radial} can drive it negative.

\subsection{Summary and comparison with measurement}

Table~\ref{tab:exponents} collects the predictions and the measured values. The
nominal Adam column uses $\expp = \expq = 1/2$, the idealized second-moment
preconditioner of Appendix~\ref{app:precond}.

\begin{table}[h]
\centering
\begin{tabular}{lcccc}
\toprule
& general & SGD ($\expq = 0$) & Adam (nominal) & Adam (measured) \\
\midrule
$\ar$   & $1/(3-\expq)$          & $1/3$ & $2/5$  & $0.48 \pm 0.03$ \\
$\at$   & $(1-\expq)/(3-\expq)$  & $1/3$ & $1/5$  & $0.08 \pm 0.01$ \\
$s_{D}$ & $-\expq/3$ & $0$ & $-1/6$ & \rev{$-0.227 \pm 0.007$} \\
$s_{B}$ & $1-\expp$ & $1$ & $1/2$ & \rev{$0.585 \pm 0.015$} \\
$\aD$   & $1/3$                  & $1/3$ & $1/3$  & $0.341 \pm 0.003$ \\
\bottomrule
\end{tabular}
\caption{Predicted and measured exponents. Every nominal entry for Adam is
displaced by the measurement, yet $\aD$ is unchanged: \rev{the measured $\ar$ corresponds to $\expq \simeq 0.92$ by inversion in the decoupled approximation, but with the feedback correction of Appendix~\ref{app:corrected} all routes give $\expq = 0.65$--$0.68$ (Table~\ref{tab:q-routes}); $s_{D}$ is from the common-teacher sweep and $s_{B}$ from the plane fit (Appendix~\ref{app:corrections}), which gives $\expp \simeq 0.42$}, a separation discussed in
Section~\ref{sec:discussion}. The final row is measured for all seven
optimizers in Figure~\ref{fig:fig4}b and Table~\ref{tab:envelope-cutoff}.}
\label{tab:exponents}
\end{table}

\begingroup\color{revcolor}
\subsection{Beyond the decoupled approximation: tangential feedback}
\label{app:corrected}

Appendix~\ref{app:radial} derived the radial equation of motion including the
tangential feedback, $d\bS/d\tau \propto 1 + R$, and Appendix~\ref{app:precond} its
preconditioned form, Eq.~\eqref{eq:corrected-drift}; Appendices~\ref{app:precond}--\ref{app:envelope}
then worked in the decoupled approximation, $R \to 0$. This subsection treats the feedback,
together with the dependence of the gradient noise on misalignment. Neither
changes the scaling symmetry behind the optimal envelope, so $\aD = 1/3$ and
$s_{D} = -\expq/3$ survive up to finite-norm corrections of order $1/\bS$, while the trajectory
exponents of optimizers with $\expq > 0$ and the channel ratio at the optimum
shift.

\subsubsection{Noise that grows with misalignment}

The per-sample gradient covariance itself depends on the misalignment: a
misaligned student has more samples near its own decision boundaries. Within
the boundary layer, with the student's logit gap written as the teacher gap plus
a Gaussian residual of variance set by $R$, the trace in
Proposition~\ref{prop:trace} acquires the factor
\begin{equation}
g(R) = \frac{\mathcal{B}(\Delta)}{\mathcal{B}(0)}, \qquad
\mathcal{B}(\Delta) = \E_{\zeta}\!\left[2\log\!\left(2\cosh\sqrt{\Delta/2}\,\zeta\right) - 1\right],
\qquad \Delta = \frac{\pi^{2}}{6}R ,
\label{eq:gR}
\end{equation}
with $\zeta \sim \mathcal{N}(0,1)$ and $g(0) = 1$. The function $\mathcal{B}$ is
the online-noise integral of \citet{kuhn2026boundary}
(Appendix~\ref{app:corrected-kuhn}).

\subsubsection{The symmetry survives}
\label{app:corrected-sym}

Both corrections, the factor $(1 + R)$ in Eq.~\eqref{eq:corrected-drift} and the
factor $g(R)$ in the tangential channel, depend on the state only through
$R \propto \xc\bS^{\expq}$. The equations with feedback, written against the sample budget, are therefore invariant under
$\data \to \lambda\data$, $\xc \to \lambda^{-\expq/3}\xc$,
$\bS \to \lambda^{1/3}\bS$, exactly as in the argument of
Appendix~\ref{app:validity}, and the loss takes the form
$\Ltot = \data^{-1/3}\,\mathcal{G}(\xc\data^{\expq/3})$. Minimizing over $\xc$ at
fixed $\data$ gives $\xcopt \propto \data^{-\expq/3}$ and
$\Ltot^{\ast} \propto \data^{-1/3}$ whatever the form of $\mathcal{G}$. The
corrections change the scaling function, and with it the prefactor and the
channel ratio at the optimum, but not $\aD$ or $s_{D}$.

For SGD the equations with feedback contain no free parameter. The two boundary
constants are $C_{L} = (\pi^{2}/12)\,p_{\Delta}(0)$ (Eq.~\eqref{eq:Lrad}) and,
since each per-sample gradient has squared norm $\lVert q - p\rVert^{2}\lVert x\rVert^{2}$
with $\lVert x\rVert^{2} = \ndim$ and $\int_{0}^{\infty}2\sigma(-u)^{2}\,du = 2\log 2 - 1$,
$\Tr\Sig = C_{T}/\bS$ with $C_{T} = \ndim\,p_{\Delta}(0)\,(2\log 2 - 1)$. The
stationary tangential loss of Proposition~\ref{prop:trace}, with the noise factor of
Eq.~\eqref{eq:gR}, then fixes the channel ratio as a function of the optimizer
coordinate,
\begin{equation}
R = \kappa\,\xc\,g(R), \qquad \kappa = \frac{C_{T}}{4C_{L}} = \frac{3\ndim(2\log 2 - 1)}{\pi^{2}} ,
\label{eq:R-of-z}
\end{equation}
in which $p_{\Delta}(0)$ cancels ($\kappa = 3.757$ for $\ndim = 32$). The drift of Eq.~\eqref{eq:corrected-drift} gives $\bS^{3} = 3C_{L}\,\xc\,(1 + R)\,\data$, and
the loss is $\Ltot = C_{L}(1 + R)/\bS$. Eliminating $\xc$ through
Eq.~\eqref{eq:R-of-z}, the loss at fixed $\data$ is proportional to
$[(1 + R)^{2}g(R)/R]^{1/3}\data^{-1/3}$, minimized at $R^{\ast} = 0.432$, that is
$\Lt/\Ltot = 0.301$ and $\xcopt = R^{\ast}/(\kappa g(R^{\ast})) = 0.0637$; without the
two corrections the same steps give $(1 + R)^{3}/R$ and the decoupled $R^{\ast} = 1/2$, $\Lt/\Ltot = 1/3$ of Theorem~\ref{thm:envelope}. The same
equations make three further predictions with no fitted parameter.
Table~\ref{tab:sgd-predictions} compares all of them with measurement.

\begin{table}[h]
\centering
\begin{tabular}{lcc}
\toprule
quantity (SGD) & predicted & measured \\
\midrule
optimal $\xcopt = \lr^{\ast}/\bs$ & $0.0637$ & $0.0625$, $0.0645$ \\
$\Lt/\Ltot$ at the optimum & $0.301$ & $0.296$ \\
crossing $\xc_{c}$ where $G = 0$ ($R = 1$) & $0.100$ & $0.101$--$0.104$ \\
$d\log R/d\log\xc$ at the optimum & $1.66$ & $1.63$ \\
$G/\Ltot$ at the optimum ($\bS \to \infty$) & $-0.397$ & see text \\
\bottomrule
\end{tabular}
\caption{Parameter-free predictions of the SGD equations with feedback against measurement. The measured $\xcopt$ is for $\bs = 256$ and $128$; the measured $\xcopt$ and crossing combine the eight-batch and common-teacher sweeps. The slope $d\log R/d\log\xc$ tests the noise factor
$g(R)$ directly: without it, $R$ would be linear in $\xc$.}
\label{tab:sgd-predictions}
\end{table}
Because $R = \Lt/\Lr$ grows faster than linearly in $\xc$, the decoupled relation $R \propto \xc$ would misplace the crossing $G = 0$ (at $\xc = 0.266$); the measured value of $G/\Ltot$ at the optimum approaches the predicted $-0.397$ as $\bS$
grows, with the remaining difference accounted for by the measured finite-norm
term $\E[H(q)]/\Lr = 2 + c/\bS$, $c \approx 13$.

\subsubsection{Consequences for the trajectory exponents and for $q$}

Along a single run at fixed $\xc$, $R \propto \xc\bS^{\expq}$ grows with the norm
whenever $\expq > 0$, so the factor $(1 + R)$ strengthens over the run and
shifts the trajectory exponents away from their decoupled values
$\ar = 1/(3-\expq)$ and $\at = (1-\expq)/(3-\expq)$. For SGD, $R$ is constant along a run and no exponent shifts. In terms of
Equation~\eqref{eq:channels}, a run at fixed $\xc$ moves along $u = \xc\data^{\expq/3}$, so its local exponents are
$\ar = 1/3 + (\expq/3)\,b'$ and $\at = \ar - (\expq/3)\,R'$, with primes denoting $d\log/d\log u$, and
\begin{equation}
2\ar + \at = 1 + \expq\left(b' - \tfrac{1}{3}R'\right).
\label{eq:sum-rule-criterion}
\end{equation}
The sum rule therefore holds exactly when $R \propto b^{3}$, as for the power laws of the decoupled approximation; the
feedback breaks this proportionality at first order in $\expq$, while for SGD, with $\expq = 0$, the sum rule holds
whatever the shapes of $b$ and $R$. For $\expq > 0$ the correction is of order $\expq\,\Lt/\Ltot$, to which the
finite-norm slope of Remark~\ref{rem:adam-edge} adds one of order $1/\bS$; both are of the size of the few-percent excess we
measure. The optimal envelope, by contrast, holds $R$ fixed across budgets
($R^{\ast} \propto \xcopt(\bS^{\ast})^{\expq}$ is independent of $\data$), so
$s_{D} = -\expq/3$ is the clean estimator of $\expq$.

Table~\ref{tab:q-routes} shows the effect for Adam. Inverting each measured
exponent with the decoupled formulas gives values of $\expq$ that spread
from $0.68$ to $0.92$; including the feedback shift, evaluated at Adam's measured $\Lt/\Ltot$, brings all four into agreement. \rev{These values describe the window $t \geq 10^{3}$ used throughout; later in training, Adam's exponents approach the nominal $\expp = \expq = 1/2$ (Appendix~\ref{app:late-time}).}

\begin{table}[h]
\centering
\begin{tabular}{lcc}
\toprule
measured exponent & value & inferred $\expq$ \\
\midrule
norm exponent $\alpha_{\bS}$ & $0.44$ & $0.65$ \\
radial exponent $\ar$ & $0.48 \pm 0.03$ & $0.65 \pm 0.16$ \\
tangential exponent $\at$ & $0.08 \pm 0.01$ & $0.66 \pm 0.02$ \\
optimal-rate exponent $s_{D}$ & $-0.227 \pm 0.007$ & $0.68 \pm 0.02$ \\
\bottomrule
\end{tabular}
\caption{Estimates of Adam's effective $\expq$ from four routes, inverting each exponent with the feedback included (linearized about $\expq = 0.68$). The routes are not
equally informative: $\ar$ constrains $\expq$ weakly, while $\at$ and $s_{D}$ agree within $0.02$. All routes place $\expq$ above the nominal $1/2$; we do not have an
explanation for that excess.}
\label{tab:q-routes}
\end{table}
\par\endgroup

\section{Shifted exponents for a general margin density}
\label{app:margin}

Appendices~\ref{app:radial}--\ref{app:envelope} assumed a regular decision
boundary, $0 < p_{\Delta}(0) < \infty$. Here we relax that to
\begin{equation}
p_{\Delta}(\delta) \simeq \rho_{\nu}\,\delta^{\nu}
\quad \text{as } \delta \to 0^{+}, \qquad \nu > -1,
\label{eq:nu-density}
\end{equation}
so that $\nu = 0$ recovers the regular case and $\nu > 0$ describes a density
that vanishes continuously at the boundary. The point of this appendix is that the shift changes the numbers but
not the structure: the exponents still collapse onto a line free of the
optimizer, and the optimal data scaling is still optimizer-independent.

\subsection{Primitives}

Repeating the boundary-layer integral of Appendix~\ref{app:radial} with
Eq.~\eqref{eq:nu-density} and substituting $u = \bS\delta$,
\begin{equation}
\Lr(\bS) \simeq \rho_{\nu}\,\bS^{-(\nu+1)}
\int_{0}^{\infty} u^{\nu}\log\!\left(1+e^{-u}\right) du
\;\sim\; \bS^{-(\nu+1)} ,
\end{equation}
and the covariance, being the same kind of boundary-layer average with a
different kernel, carries the same power, $\Sig_{ii} \sim \bS^{-(\nu+1)}$. In
the general notation of Appendices~\ref{app:radial}
and~\ref{app:tangential}, in which the radial loss and the gradient covariance
carry independent powers of the norm,
\begin{equation}
\Lr \sim \bS^{1-\expg}, \qquad \Sig_{ii} \sim \bS^{-\exps},
\end{equation}
this reads
\begin{equation}
\expg = \nu + 2, \qquad \exps = \nu + 1 .
\end{equation}
Only these two numbers carry the data geometry; the derivations of
Appendices~\ref{app:precond} and \ref{app:tangential} are otherwise unchanged.

\subsection{Shifted exponents}

With general $\expg$ and $\exps$, the radial drift
$d\bS/d\tau \sim \bs^{\expp}\bS^{\expq-\expg}$ integrates to
$\bS \sim (\xc\data)^{1/m}$ with $m = \expg + 1 - \expq$, and the two channels
become
\begin{equation}
\Lr \sim (\xc\data)^{-\ar}, \qquad
\Lt \sim \xc^{\,\kappa\ar}\,\data^{-\at},
\end{equation}
with
\begin{equation}
\ar = \frac{\expg-1}{m}, \qquad
\at = \frac{\exps-\expq}{m}, \qquad
\kappa = \frac{\expg+1-\exps}{\expg-1} .
\end{equation}
For $\nu = 0$ this gives $\kappa = 2$ and recovers
Eq.~\eqref{eq:channels-app}. Eliminating $\expq$ exactly as in
Appendix~\ref{app:envelope} leaves the shifted sum rule
\begin{equation}
\boxed{\;\kappa\,\ar + \at = 1, \qquad
\kappa = \frac{2}{\nu+1}. \;}
\label{eq:sumrule-nu}
\end{equation}
Minimizing $\Lr + \Lt$ over $\xc$ gives $\Lr^{\ast} = \kappa\,\Lt^{\ast}$ and,
on the line Eq.~\eqref{eq:sumrule-nu},
\begin{equation}
\boxed{\;\aD = \frac{1}{1+\kappa} = \frac{\nu+1}{\nu+3}. \;}
\label{eq:aD-nu}
\end{equation}
The optimal learning rate scales as $\lropt \propto \data^{-\expq/(\nu+3)}$ at
fixed batch size, while $s_{B} = 1 - \expp$ is unchanged, since it comes from
the definition of $\xc$ rather than from the data geometry.

\subsection{What survives the shift}

Table~\ref{tab:nu} collects the results. Three things are preserved.
First, $\expq$ drops out of both
Eq.~\eqref{eq:sumrule-nu} and Eq.~\eqref{eq:aD-nu}, so the sum rule remains a
line in the $(\ar, \at)$ plane that every optimizer must lie on, and the optimal
data scaling remains optimizer-independent. Second, the relation between the
loss ratio at the optimum and the optimal exponent persists: since
$\Ltot = (1+\kappa)\Lt^{\ast}$ there,
\begin{equation}
\Lt/\Ltot = \frac{1}{1+\kappa} = \aD \qquad \text{at the optimum,}
\end{equation}
so the tangential fraction of the loss always equals the optimal data-scaling
exponent, of which $\Lt/\Ltot = \aD = 1/3$ is the regular-boundary case. Third,
$\aD$ is increasing in $\nu$, so a density that vanishes at the boundary
accelerates the optimum without breaking optimizer-independence.

\begin{table}[h]
\centering
\begin{tabular}{lccc}
\toprule
& $\nu = 0$ & $\nu = 1$ & general $\nu$ \\
\midrule
sum rule & $2\ar + \at = 1$ & $\ar + \at = 1$
& $\tfrac{2}{\nu+1}\ar + \at = 1$ \\
$\aD$ & $1/3$ & $1/2$ & $(\nu+1)/(\nu+3)$ \\
$\Lr^{\ast}/\Lt^{\ast}$ & $2$ & $1$ & $2/(\nu+1)$ \\
\bottomrule
\end{tabular}
\caption{Effect of the margin density on the sum rule and the optimum. The
exponents shift with $\nu$, but none of these quantities depends on the
optimizer.}
\label{tab:nu}
\end{table}

\begin{remark}
The acceleration has a limit: $\alpha_D \to 1$ as $\nu \to \infty$. Beyond it lies a hard margin gap $\delta_{\min} > 0$, which Eq.~\eqref{eq:nu-density} does not describe. There the loss and the radial gradient both decay as $e^{-\bS\delta_{\min}}$, so the drift itself is exponentially weak: under gradient descent the norm grows only logarithmically, $\bS \simeq \delta_{\min}^{-1}\log\tau$, and the loss falls as $\tau^{-1}$ up to logarithmic factors, the rate known for separable data \citep{soudry2018implicit, ji2018risk} and the $\alpha_D \to 1$ limit above. The power-law regime studied here is therefore the consequence of samples arriving arbitrarily close to the decision boundary, and any mechanism that thins their supply accelerates training.
\end{remark}

\section{Validity conditions}
\label{app:validity}

\subsection{Two alternatives ruled out by the sum-rule line}

Because Eq.~\eqref{eq:sumrule} contains no fitted quantity, it also serves as a
test: mechanisms that would produce power-law channels by a different route
generally land off the line. Two are worth naming.

\paragraph{A non-decaying covariance sector.}
Suppose part of the gradient covariance does not decay with $\bS$. Partition the
active coordinates into a set with $\Sig_{ii} \to h_{i} > 0$ and the rest with
$\Sig_{ii} \simeq s_{i}/\bS$. The trace identity of
Proposition~\ref{prop:trace} then gives a tangential loss of the form
\begin{equation}
\E[\Lt] = \underbrace{\frac{\lr}{4\bs}\sum_{i \in \mathrm{hard}}A_{ii}h_{i}}
_{\propto\, \xc,\ \text{independent of } \data}
\;+\; \text{(decaying bulk)} ,
\end{equation}
that is a constant ceiling plus a decaying remainder. A ceiling means
$\at = 0$, and because it places no constraint on the radial channel, $\ar$
would be free: points would lie anywhere along the horizontal line $\at = 0$
rather than on the sum rule.

This alternative deserves naming because Adam's measured $\at = 0.08$ is close
to zero, so the tangential data alone are superficially consistent with it. What
separates the two is $\ar$. A hard sector leaves $\ar$ unconstrained, whereas the
sum rule fixes it at $(1 - \at)/2 = 0.46$; the measured value is $0.48$. We
note also that the realizable hard-label teacher of Appendix~\ref{app:radial}
provides no source for such a sector, which requires a non-vanishing probability
of exactly zero margin, that is genuinely ambiguous labels.

\begingroup\color{revcolor}
\paragraph{Noise-driven norm growth.} Stochastic tangent updates inflate the norm in addition to
the drift. Relative to the drift-driven growth, this contribution scales as
\begin{equation}
R_{\mathrm{rad}} \propto \lr\,\bs^{\expp-1}\bS^{\expq} = \xc\,\bS^{\expq} ,
\label{eq:Rrad}
\end{equation}
the same combination as the channel ratio $R$. For SGD it is independent of $\bS$, and
Appendix~\ref{app:radial} includes it in the radial equation of motion, where
together with the drift it makes the norm grow with the full loss. For adaptive
optimizers it grows with the norm, so the drift dominates in a window that closes at
large $\bS$. Beyond that window the noise term would set the growth,
$\bS \sim \tau^{1/(3-2\expq)}$, which for $\expq = 1/2$ gives $(\ar, \at) = (1/2, 1/4)$
and $2\ar + \at = 5/4$, off the line; the measured points lie on it, and
Appendix~\ref{app:drift} confirms directly that the noise term is a few percent of the
drift for every adaptive optimizer.
\par\endgroup

\begingroup\color{revcolor}
\subsection{Drift and diffusion, measured}
\label{app:drift}

Between consecutive checkpoints each run logs the summed drift and noise
contributions to the change in $\bS^{2}$,
\[
\textstyle\sum 2\langle \WS, \Delta\WS\rangle_{F} \qquad\text{and}\qquad \sum \lVert \Delta\WS \rVert^{2},
\]
which together equal that change exactly, and also $\bS$, $\omega$ and the losses. We compare each
term with its prediction over the late intervals ($t \geq \max(0.05\,T, 200)$)
of the two largest budgets, at the three learning rates around each run's
optimum, pooling the five teachers (constant rate, no warmup).
Table~\ref{tab:drift} gives the median ratios for SGD.

\begin{table}[h]
\centering
\begin{tabular}{lcc}
\toprule
quantity & median ratio & interquartile range \\
\midrule
drift: $\sum 2\langle \WS,\Delta\WS\rangle \,/\, (-2\lr G\, n)$ & $1.006$ & $[0.990, 1.021]$ \\
noise: $\sum \lVert\Delta\WS\rVert^{2} \,/\, (4\lr\Lt\, n)$ & $1.002$ & $[0.979, 1.021]$ \\
norm: $\Delta\bS^{2} \,/\, (2\lr\Ltot\, n)$ & $1.029$ & $[1.017, 1.041]$ \\
overlap: $\Delta\omega^{2} \,/\, (2\lr\Ltot\, n)$ & $1.029$ & $[1.017, 1.041]$ \\
noise / drift in $\bS^{2}$ (measured) & $1.25$ & \\
noise / drift in $\bS^{2}$ (predicted, $4\Lt/(-2G)$) & $1.26$ & \\
\bottomrule
\end{tabular}
\caption{Drift and diffusion for SGD; $n$ is the number of steps in an interval.
The noise term matches the trace identity within $0.2\%$, and the $3\%$ excess
in the norm and overlap rows is the finite-norm correction of
Appendix~\ref{app:corrected-sym}.}
\label{tab:drift}
\end{table}

Two conclusions follow. First, in $\bS$ the noise term is larger than the drift
(about $1.25$ times at these rates), so for SGD the norm growth is not drift-dominated in $\bS$; in $\omega$, growth is drift
alone and matches the full loss. Second, for the adaptive optimizers the noise
term is small, $1$--$3\%$ of the drift for Adam, Muon and the PowerAdam family
and $6\%$ for SignGD, so their norm growth is drift-dominated even in $\bS$; the
equality $\Delta\omega^{2} = \Delta\bS^{2}$, within $0.1\%$ for every optimizer
except SignGD ($1.4\%$), shows that their tangential distribution is stationary
late in training. Since their preconditioner's scale is not separable from
these logs, we do not test the full-loss force for adaptive optimizers directly;
for them, the effective description of Appendix~\ref{app:precond} applies.
\par\endgroup

\subsection{Regime conditions}

\paragraph{Noise-dominated second moment.}
Appendix~\ref{app:precond} used $V_{i} \simeq \Sig_{ii}/\bs$, dropping the
squared mean gradient in $V_{i} = \bar{g}_{i}^{2} + \Sig_{ii}/\bs$. Boundary
localization makes $\bar{g}_{i}$ decay as $\bS^{-2}$ while
$\Sig_{ii}/\bs \sim 1/(\bs\bS)$, so the neglected ratio is
$\bar{g}_{i}^{2}/(\Sig_{ii}/\bs) \sim \bs\bS^{-3}$ and the approximation
requires $\bS^{3} \gg \bs$. With $\bs \leq 2048$ and $\bS$ reaching $10^{2}$ or
more in our runs, this is satisfied by a wide margin over the fitted range.

\paragraph{Finite-step stability.}
The updates analyzed here are discrete, and the preconditioned drift operator
must have eigenvalues inside the discrete stability region. In the
optimizer-metric coordinates of Appendix~\ref{app:tangential} a sufficient
condition is
\begin{equation}
0 < \lr\,\lambda_{\max}\!\left(A^{1/2}\Hess A^{1/2}\right) < 2 .
\end{equation}
This bounds the usable learning rate from above. Runs that violate it leave the
scaling regime altogether rather than shifting exponents, so they sit far from
the optimum and are never selected when the optimal learning rate is located; no
explicit exclusion is therefore applied (Appendix~\ref{app:experiments}).

\subsection{Features of Adam absent from the analysis}

Three features of Adam as implemented do not appear in the update
$\Delta\WS = -\lr A\hat{g}$ analyzed above.

\paragraph{Finite $\epsilon$.}
The preconditioner is $(\sqrt{V} + \epsilon)^{-2\pwr}$ rather than $V^{-\pwr}$,
and saturates once $\sqrt{V_{i}} \ll \epsilon$. Since
$V_{i} \simeq s_{i}/(\bs\bS)$, this occurs beyond
$\bS_{\epsilon} \sim s_{i}/(\bs\epsilon^{2})$, past which $A$ is an
$\epsilon$-dependent constant and the dynamics becomes SGD-like up to a
rescaling of the learning rate, returning the radial exponent to $1/3$. The
adaptive window is therefore bounded above by $\epsilon$ as well as by
Eq.~\eqref{eq:Rrad}. We state this as a prediction and do not test it here.

\paragraph{Second-moment memory.}
With $v_{t} = \beta_{2}v_{t-1} + (1-\beta_{2})\hat{g}_{t}^{\odot 2}$ the
preconditioner tracks $V$ with a lag $t_{v} \sim 1/(1-\beta_{2})$, and the
instantaneous treatment is valid when
$\left(1-\beta_{2}\right)^{-1}\lvert d\log V_{i}/dt\rvert \ll 1$. Along a
power-law trajectory $d\log V_{i}/dt = O(t^{-1})$, so the condition improves at
late times for every fixed $\beta_{2} < 1$. This is the concrete form of the
adiabaticity claim in Appendix~\ref{app:tangential}.

\paragraph{Momentum.}
Appendix~\ref{app:tangential} assumes a first-order Markov recursion in $e$.
With a first moment the state is augmented to $s_{t} = (e_{t}, m_{t-1})$,
evolving as $s_{t+1} = F s_{t} + G\xi_{t}$, and the stationary covariance
satisfies $C_{s} = F C_{s}F^{\top} + \tfrac{1}{\bs}G\Sig G^{\top}$. This is
preferable to replacing $\bs$ by an effective batch size, because momentum
introduces temporal correlations as well as a change in instantaneous variance.
Since the stationary mean of $m_{t}$ is the slowly varying mean gradient,
momentum does not change the leading $\bS$-dependence of the radial drift in the
adiabatic regime, although it can change tangential prefactors and the stability
boundary above. We do not re-derive Proposition~\ref{prop:trace} in the
augmented state. The empirical evidence that the scaling form survives is in
Figure~\ref{fig:fig4}a: Adam, the three PowerAdam variants and SignGD all carry
a first moment, and all lie on the line.

\subsection{Conditions established elsewhere}

For convenience, the remaining conditions and where they are stated: a regular
decision boundary (Appendix~\ref{app:radial}, relaxed in
Appendix~\ref{app:margin}); a stable preconditioned radial shape
(Appendix~\ref{app:precond}); the local quadratic regime, stationarity of the
tangent distribution, small steps, $A$ approximately constant over the local
equilibration time, and adiabaticity (Appendix~\ref{app:tangential}); and hard
labels, without which the student has a finite optimal scale and the power law
holds over an intermediate range rather than indefinitely
(Section~\ref{sec:discussion}).

\section{Experimental details}
\label{app:experiments}

\subsection{Model and sweeps}

All runs use the teacher--student model of Section~\ref{sec:model} and
Appendix~\ref{app:radial}: $\ncls = 128$ classes, $\ndim = 32$ input dimensions,
inputs drawn i.i.d.\ from $\mathcal{N}(0, I_{\ndim})$ and RMS-normalized, and one-hot labels from a
fixed unit-norm teacher. Training is online, with a fresh batch drawn at every
step and no sample reused. Every run is $10{,}000$ steps, logged every $10$
steps.

The student is initialized at $\WS = 0$, which lies in the row-centered subspace of Appendix~\ref{app:radial}, and the teacher's entries are drawn i.i.d.\ from PyTorch's default linear-layer initialization, $\mathcal{U}(-m^{-1/2}, m^{-1/2})$. Adam and PowerAdam use $(\beta_1, \beta_2, \epsilon) = (0.9, 0.999, 10^{-8})$, Muon uses Nesterov momentum $0.95$ with five Newton--Schulz iterations, and SignGD uses momentum $0.9$; no optimizer uses weight decay. At each logged step the test loss is evaluated on a fresh batch of $4096$ samples\rev{; these are the eight-batch sweeps used for Figures~\ref{fig:fig1}--\ref{fig:fig4}, while the common-teacher sweeps of Appendix~\ref{app:common-teacher} use five fixed teachers, shared data streams across learning rates, and a fixed margin-stratified evaluation set}. The rescaled teacher $W_R = \bS\widehat{W}_T$ is realized on logits, by rescaling the teacher's logits to the student's mean logit standard deviation across classes: for a row-centered matrix with $\mathbb{E}[xx^\top] = I_m$ this standard deviation is $\|W\|_F/\sqrt{n}$, as we confirm for the student (measured $0.0886\,\bS$ against $\bS/\sqrt{128} = 0.0884\,\bS$), and it is unaffected by a common logit shift. $L$ and $\Lr$ are evaluated on the same test batch, so $\Lt = L - \Lr$ does not accumulate independent sampling noise from its two terms.

For each optimizer we sweep $16$ logarithmically spaced learning rates against
$8$ batch sizes, $\bs \in \{16, 32, 64, \ldots, 2048\}$, for $128$ runs per
optimizer. The learning-rate ranges differ by optimizer, because the optimum
sits at very different scales (Table~\ref{tab:sweeps}):

\begin{table}[h]
\centering
\begin{tabular}{ll}
\toprule
optimizer & learning-rate range \\
\midrule
SGD                     & $[10^{-1},\ 10^{3}]$ \\
PowerAdam($0.125$)      & $[10^{-1},\ 10^{3}]$ \\
PowerAdam($0.25$)       & $[10^{-2},\ 10^{2}]$ \\
PowerAdam($0.375$)      & $[10^{-2},\ 10^{1}]$ \\
Adam $=$ PowerAdam($0.5$) & \rev{$[10^{-3},\ 10^{0}]$} \\
Muon                    & $[10^{-2.5},\ 10^{1}]$ \\
SignGD                  & $[10^{-3},\ 10^{0}]$ \\
\bottomrule
\end{tabular}
\caption{Learning-rate sweep ranges, each with $16$ logarithmically spaced
values. The ranges shift downward as the preconditioner power increases,
spanning six decades across the family; this is the unquantified form of the
optimizer-dependence measured in Figure~\ref{fig:fig3}.}
\label{tab:sweeps}
\end{table}

\subsection{Locating the optimum}

For each batch size and each logged step we first take the learning rate that
minimizes the test loss on the grid. Because the grid is coarse, we then refine
it: writing the loss near its minimum as a quadratic in $\log_{10}\lr$, we fit a
parabola through the grid minimum and its two neighbors and take its vertex,
\begin{equation}
\lropt = 10^{-b/(2a)}, 
\end{equation}
for a fit $\Ltot \simeq a(\log_{10}\lr)^{2} + b\log_{10}\lr + c$.

The same quadratic structure that justifies this interpolation also explains an
asymmetry visible in Figure~\ref{fig:fig3}: near the optimum the location of the
minimum is determined only at second order in $\log_{10}\lr$, while its value is
determined at first order. The $\lropt$ traces in panel (a) are therefore noisy
at large $\data$ while the envelope in panel (c) is clean.

Runs that violate the stability bound of Appendix~\ref{app:validity} have large
loss, sit far from the optimum, and are never selected by this procedure, so no
learning rates are excluded by hand.

\begin{figure}[h]
\centering
\includegraphics[width=\linewidth]{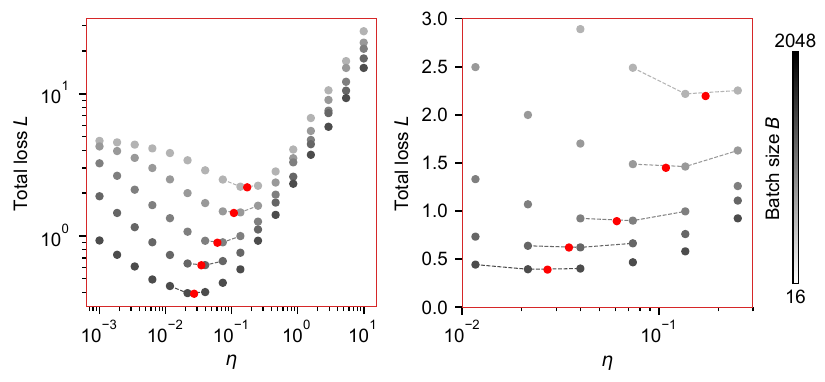}
\caption{\textbf{The loss basin is broad and quadratic in $\log\lr$.} Total
loss against learning rate for Adam, one point per swept learning rate, with
grey level denoting batch size at a fixed training step. Red markers are the
parabola vertices used as $\lropt$; they fall between grid points because they
are interpolated. \textbf{Left:} the full swept range. \textbf{Right:} the same
data near the minimum, on a linear vertical axis. Two things are visible. The
minima are interior to the swept range at every batch size, so the estimator
never has to extrapolate to a grid edge. And the basin is flat: the loss changes
little over a factor of several in $\lr$. Since the loss is quadratic in
$\log\lr$ near its minimum, the location of the optimum is fixed only at second
order while its value is fixed at first, which is why the $\lropt$ traces in
Figure~\ref{fig:fig3}a are noisy at large $\data$ while the envelope in
Figure~\ref{fig:fig3}c is clean.}
\label{fig:fs3a}
\end{figure}

\subsection{The plane fit}

The optimal learning rate is fitted jointly across batch sizes and budgets as a
log-log plane,
\begin{equation}
\log_{10}\lropt = s_{D}\log_{10}\data + s_{B}\log_{10}\bs + c,
\qquad \data = \bs t ,
\end{equation}
by ordinary least squares over all pairs $(\bs, t)$ with $t \geq 10^3$ steps, matching the window used for the channel exponents. The uncertainties quoted for
$s_{D}$ and $s_{B}$ are the spread across four lower budget cutoffs, $D \ge \{10^5, 3\times 10^5, 10^6, 3\times 10^6\}$, rather than a regression standard error, since the fitted points are correlated across training steps. 

The preconditioner exponents follow from Appendix~\ref{app:envelope}:
$\expq = -3\,s_{D}$ and $\expp = 1 - s_{B}$. For a general margin density the
first becomes $\expq = -(\nu+3)\,s_{D}$ (Appendix~\ref{app:margin}).

Fitting the plane rather than each batch size separately matters. Slopes taken
from constant-step slices differ from those taken from constant-budget slices,
and the joint fit removes that ambiguity by using every $(\bs, t)$ pair at once.

\begingroup\color{revcolor}
\label{app:corrections}

Table~\ref{tab:corrected-sD} gives the optimal-learning-rate exponents for all seven optimizers, from this plane fit, together with $s_{D}$ from the common-teacher sweep at fixed batch size (Appendix~\ref{app:common-teacher}). The two procedures differ in seeding, evaluation, batch range and range of $\data$, and agree in $s_{D}$ to within about $0.02$. Only the eight-batch sweeps give $s_{B}$.

\begin{table}[h]
\centering
\begin{tabular}{lccc}
\toprule
& \multicolumn{2}{c}{$s_{D}$} & $s_{B}$ \\
\cmidrule(lr){2-3}
optimizer & plane fit & common-teacher & plane fit \\
\midrule
SGD & $-0.013 \pm 0.004$ & $-0.014 \pm 0.005$ & $0.972 \pm 0.018$ \\
PowerAdam($0.125$) & $-0.064 \pm 0.003$ & $-0.087 \pm 0.006$ & $0.896 \pm 0.006$ \\
PowerAdam($0.25$) & $-0.132 \pm 0.005$ & $-0.131 \pm 0.007$ & $0.799 \pm 0.016$ \\
PowerAdam($0.375$) & $-0.181 \pm 0.005$ & $-0.181 \pm 0.008$ & $0.667 \pm 0.007$ \\
Adam & $-0.249 \pm 0.004$ & $-0.227 \pm 0.007$ & $0.585 \pm 0.015$ \\
Muon & $-0.212 \pm 0.005$ & $-0.238 \pm 0.006$ & $0.570 \pm 0.004$ \\
SignGD & $-0.213 \pm 0.014$ & -- & $0.565 \pm 0.031$ \\
\bottomrule
\end{tabular}
\caption{Optimal-learning-rate exponents. ``Plane fit'' applies this procedure to the eight-batch sweeps; ``common-teacher'' is the fixed-batch sweep with five teachers. SignGD's
fixed-batch value is omitted because its constant-rate envelope does not reach
$1/3$ at fixed batch (Appendix~\ref{app:signgd-fixed}).}
\label{tab:corrected-sD}
\end{table}

Across the PowerAdam family $s_{B}$ falls monotonically,
with $\expp = 1 - s_{B}$ close to the nominal power ($\expp \approx 0.85\pwr$);
$\expq$ exceeds the nominal power by $0.14$--$0.18$ in the common-teacher values.
\par\endgroup

\subsection{Exponent fits and uncertainties}

Time-scaling exponents are obtained by ordinary least squares on $\log$ loss
against $\log$ training step over a fixed late-time window, applied identically
to every curve and every optimizer; no window is chosen per optimizer. The value
quoted for an exponent is the mean over the four learning-rate curves shown in
the corresponding figure, and the quoted uncertainty is their standard
deviation; the crosses in Figure~\ref{fig:fig4}a are the same quantity. Reporting the spread across curves rather than a single
regression standard error is the more conservative choice, since it measures
whether the curves agree with each other and not merely how well each is fitted.
For SGD this gives a loss
exponent $0.34 \pm 0.01$ and a norm exponent $0.33 \pm 0.01$; for Adam the
norm exponent is $0.44 \pm 0.04$, the larger spread reflecting that Adam's
curves are less nearly parallel.

\paragraph{The envelope fit.}
The envelope $\Ltot^{\ast}(\data)$ is the minimum of the loss over hyperparameters
at each budget. 
At each $(\bs, t)$ the loss entering the fit is the minimum over the learning-rate grid; the parabolic refinement of Appendix~H.2 is used to locate $\lropt$ but not to evaluate $\Ltot^\ast$, which suffices because near its minimum the loss depends on $\log \lr$ only at second order.
Each batch size covers only $\data \in [10^{3}\bs, 10^{4}\bs]$, so
the number of curves available to that minimum falls from four to one across the full range of budgets, and a minimum taken over a set that shrinks with budget is biased by
an amount that varies with $\data$, which tilts the fitted slope. Because the loss
collapses onto $z \equiv \lr\bs^{\expp-1}$, the batch sizes after optimizing $\lr$
estimate one quantity rather than genuinely different conditions, so we estimate
the exponent by pooling them, with a single slope in $\log_{10}\data$ and one free
level per batch size,
\begin{equation}
\log_{10} \Ltot^{\ast}(\bs, \data) = -\alpha_D \log_{10}\data + c_{\bs},
\end{equation}
fitted by ordinary least squares over every $(\bs, \data)$ point in the window,
with each batch size weighted equally since points within a curve are correlated
across training steps. This measures the same exponent without the selection bias
of the minimum. The fitted levels $c_{\bs}$ agree to within $1\%$ across the batch
sizes that span the window for every optimizer except SignGD, an independent check
of the collapse onto $z$.

\begin{table}[t]
\centering
\begin{tabular}{lccccc}
\toprule
& \multicolumn{4}{c}{$D_{\min}$} & \\
\cmidrule(lr){2-5}
Optimizer & $10^{5}$ & $3\times10^{5}$ & $10^{6}$ & $3\times10^{6}$ & $\alpha_D$ \\
\midrule
SGD               & 0.343 & 0.342 & 0.340 & 0.341 & $0.341 \pm 0.001$ \\
PowerAdam(0.125)  & 0.334 & 0.332 & 0.332 & 0.333 & $0.333 \pm 0.001$ \\
PowerAdam(0.25)   & 0.335 & 0.334 & 0.331 & 0.333 & $0.333 \pm 0.002$ \\
PowerAdam(0.375)  & 0.337 & 0.336 & 0.336 & 0.339 & $0.337 \pm 0.001$ \\
Adam              & 0.339 & 0.339 & 0.342 & 0.347 & $0.341 \pm 0.003$ \\
Muon              & 0.336 & 0.332 & 0.330 & 0.328 & $0.332 \pm 0.003$ \\
SignGD            & 0.296 & 0.314 & 0.325 & 0.327 & $0.315 \pm 0.012$ \\
\bottomrule
\end{tabular}
\caption{\textbf{Envelope exponents against the lower fitting cutoff.} $\alpha_D$ fitted over $D \geq D_{\min}$ for four choices of $D_{\min}$, with the quoted value the mean across the four and the uncertainty their spread. The envelope is built from checkpoints with $t \geq 10^3$, the same window used for the dynamic exponents $\alpha_r$ and $\alpha_t$, so that all four exponents are fitted over matched data. The envelope is defined as the minimum of the loss over hyperparameters at each budget, as in Figure~\ref{fig:fig3}c, with its exponent estimated by a pooled fit with a single slope in $\log D$ and one free level per batch size (Appendix~\ref{app:experiments}.4). Six of the seven lie within 2.5\% of $1/3$, with a scatter across optimizers of $0.004$; \rev{SignGD reaches $1/3$ only in this envelope over batch size: its per-batch-size slopes, fitted separately, rise from $0.06$ at $B=16$ to $0.33$ at $B=2048$, and at fixed batch size its exponent falls further with $\data$ (Appendix~\ref{app:signgd-fixed}).} We stop at $D_{\min} = 3\times10^6$ because the number of batch sizes reaching a given budget falls as $D$ grows: since each covers $D \in [10^3B, 10^4B]$, fewer than three contribute above $5\times10^6$ and only one above $10^7$, so later cutoffs measure the largest batch sizes rather than a minimum over hyperparameters.}
\label{tab:envelope-cutoff}
\end{table}

\begingroup\color{revcolor}
\paragraph{SignGD at fixed batch size.}\label{app:signgd-fixed}

Table~\ref{tab:envelope-cutoff} measures SignGD's envelope over learning rate
and batch size jointly, where it approaches $1/3$. At fixed batch size the
common-teacher sweep, with the learning-rate grid extended down to $10^{-5}$ so
that every optimum is interior, gives $\aD = 0.188 \pm 0.002$ over all budgets and
$0.080 \pm 0.006$ over the final decade at $\bs = 64$, and $0.280 \pm 0.003$ and
$0.240 \pm 0.009$ at $\bs = 256$. For SignGD the one-third optimum therefore
requires tuning the batch size jointly with the learning rate; the optimal
batch size grows with the budget.
\par\endgroup

\paragraph{Consistency of the optimal exponent.}
The optimal data-scaling exponent can be obtained two ways. Fitting the envelope of Figure~\ref{fig:fig3}c gives the values of Table~\ref{tab:envelope-cutoff}, which lie within 2.5\% of $1/3$ for six of the seven optimizers, with a scatter across optimizers of $0.004$; SignGD is the exception and is discussed above. Substituting the measured channel exponents into $\alpha_D = (2\alpha_r + \alpha_t)/3$ instead gives $0.34 \pm 0.01$ for SGD and $0.35 \pm 0.02$ for Adam; this route returns the measured sum rule divided by three, so its excess and the sum rule's $2\alpha_r + \alpha_t = 1.03$ for SGD and $1.04$ for Adam are the same numbers. Both are consistent with the fitted envelopes, and both sit above $1/3$.

\paragraph{Where the excess lies.}
Writing $\ar = \alpha_\bS s_r$ and $\at = \alpha_\bS s_t$, where $s_r$ and $s_t$ are the local slopes of $\Lr$ and $\Lt$ against $\bS$, the sum rule takes the time-free form $2s_r + s_t = 1/\alpha_\bS$, in which every quantity is a slope against the weight norm. For SGD the left side measures $3.16$ against $3.07$ on the right, accounting for the whole of the measured excess. The tangential slope sits at its asymptotic value ($s_t = 1.01$ against $1$) while the radial slope exceeds its own ($s_r = 1.08$ against $1$), so the excess lies in the radial channel, and it is the $O(1/\bS)$ correction of Remark~\ref{rem:adam-edge}: $s_r$ falls monotonically with $\bS$ over the fitting window and takes nearly the same values under Adam ($1.09$). Because the two exponents are fitted to the same runs, their spreads move together, and the ratio $s_r$ is better determined than the individual uncertainties on $\alpha_\bS$ and $\ar$ suggest.

The same correction accounts for the envelope. The channel combination gives $(2\ar + \at)/3 = 0.34\pm 0.01$ for SGD and $0.35 \pm 0.02$ for Adam, and the fitted envelopes give $0.341$ for both, so the two routes agree to within their uncertainties and both exceed $1/3$ by a few percent. Estimating the envelope exponent by fitting the minimum curve directly would bias the slope downward, since that minimum is taken over a batch-size set that shrinks as $D$ grows, so the bias varies with the fitting window; the pooled estimator avoids this, and the per-optimizer spread across cutoffs is $0.003$ or less for six of the seven, and $0.012$ for SignGD. Since $s_r$ decreases with $\bS$, the framework predicts that the remaining excess shrinks at later budgets, and we read $1/3$ as the asymptotic value with $0.333 \pm 0.008$ across all seven, or $0.336 \pm 0.004$ excluding SignGD.

\begingroup\color{revcolor}
\subsection{Common-teacher sweeps}
\label{app:common-teacher}

The common-teacher sweeps, used for Figure~\ref{fig:fig5} and in Appendices~\ref{app:drift}, \ref{app:corrections}, \ref{app:signgd-fixed}, \ref{app:schedules} and~\ref{app:aniso}, use the same model as the eight-batch sweeps of Appendix~\ref{app:experiments}, with three changes in procedure. Each configuration is run for five
teachers; within a teacher, all learning rates share the teacher and the
stream of training samples, so the optimal learning rate is located without
teacher-to-teacher noise. Each budget is a separate run trained to its end. And
the loss is evaluated on a fixed set of $16{,}384$ inputs per teacher,
stratified by the teacher's margin (four strata at the $0.1\%$, $1\%$ and
$10\%$ quantiles, $4{,}096$ inputs each, weighted by their population share),
which estimates the population loss without bias and with much lower variance
at small loss.
\par\endgroup

\subsection{Supporting measurements}

\begin{figure}[h]
\centering
\includegraphics[width=\linewidth]{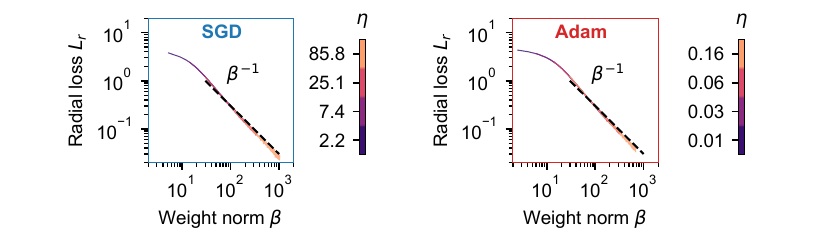}
\caption{\textbf{The radial loss is a function of the weight norm alone.}
$\Lr$ against $\bS$ for SGD (left) and Adam (right), colored by learning rate
over the ranges of Figure~\ref{fig:fig1}. Curves at different learning rates
collapse onto a single master curve approaching $\Lr \propto \bS^{-1}$
(dashed). This tests Eq.~\eqref{eq:Lrad} directly, without comparing two
separately fitted exponents.}
\label{fig:fs2a}
\end{figure}

\begin{figure}[h]
\centering
\includegraphics[width=\linewidth]{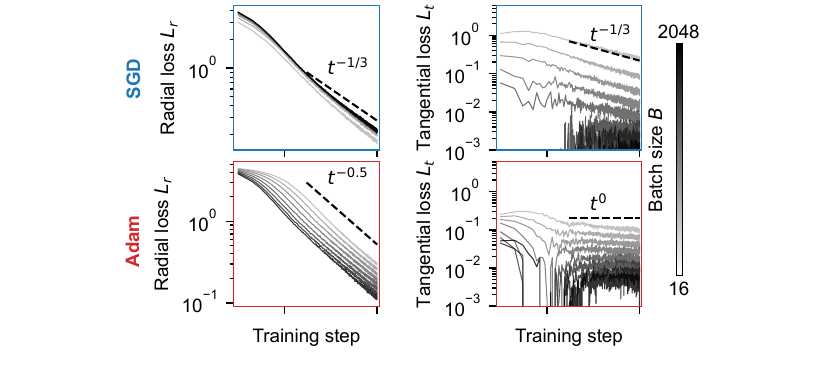}
\caption{\textbf{The channel exponents do not depend on batch size.} $\Lr$
(left) and $\Lt$ (right) against training step for all eight batch sizes at
fixed learning rate, SGD above and Adam below, with the exponents of
Figure~\ref{fig:fig2} shown as dashed guides. Batch size shifts the curves
without changing their slopes, which together with Figure~\ref{fig:fig2}
establishes that $\ar$ and $\at$ depend on neither hyperparameter. Larger
batches lower $\Lt$, while $\Lr$ is nearly batch-independent for SGD: at fixed
step $\Lr$ depends on batch size only through $\bs^{\expp}$, and $\expp = 0$ for
SGD. At a fixed sample budget $\data = \bs t$ the comparison changes sign, since
a larger batch is less noisy and lowers $\Lt$ but buys fewer optimization steps
and so leaves $\Lr$ higher. The erratic early-time behavior at small batch size
is $\Lt < 0$, the finite-norm transient described in
Appendix~\ref{app:radial}.}
\label{fig:fs2b}
\end{figure}

\begin{figure}[h]
\centering
\includegraphics[width=\linewidth]{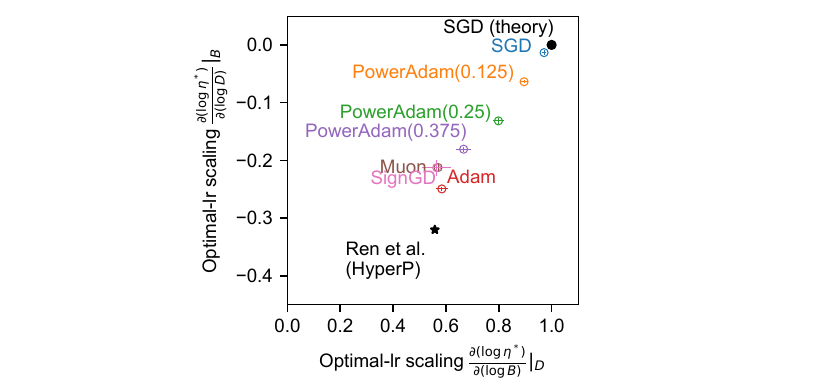}
\caption{\textbf{The two optimal-learning-rate exponents across optimizers.}
Measured $s_{D}$ against $s_{B}$, with the SGD prediction $(1, 0)$ as a black
point and the values reported by \citet{ren2026hypersphere} for their
hypersphere optimizer as a star. Since $s_{D} = -\expq/3$ and
$s_{B} = 1 - \expp$, the horizontal axis measures the preconditioner's batch
exponent and the vertical axis its scale exponent.}
\label{fig:fs4b}
\end{figure}

\begin{figure}[h]
\centering
\includegraphics[width=\linewidth]{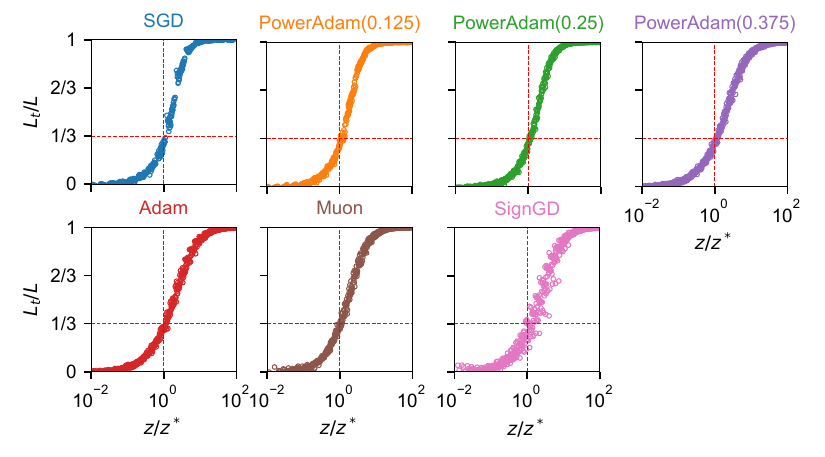}
\caption{\textbf{The two channels reach a fixed ratio at the optimum.}
$\Lt/\Ltot$ against $\xc/\xcopt$ for each of the seven optimizers, with
$\xc = \lr\bs^{\expp-1}$ and $\xcopt(\data) = 10^{c}\data^{s_{D}}$ taken from
the plane fit of Appendix~\ref{app:experiments}. Each panel contains all $16$
learning rates and all $8$ batch sizes at five training times
($t \approx 300$, $1000$, $3000$, $5000$, $9000$), and they collapse onto a
single curve: $\Lt/\Ltot$ depends on the hyperparameters only through
$\xc/\xcopt$. Every optimizer passes \rev{close to} $1/3$ at $\xc = \xcopt$ (dashed lines), as Theorem~\ref{thm:envelope} predicts \rev{in the decoupled approximation} with no fitted quantity beyond the plane fit that sets $\xcopt$ \rev{(including the feedback of Appendix~\ref{app:corrected}, the predicted value for SGD is $0.30$, within the scatter)}. The
collapse follows $R/(1+R)$ with $R = \tfrac{1}{2}(\xc/\xcopt)^{3\ar}$;
deviations appear only for $\Lt/\Ltot \lesssim 0.1$, where $\Lt$ is a small
difference of two larger quantities and the finite-norm transient of
Appendix~\ref{app:radial} can drive it negative.}
\label{fig:fs4c}
\end{figure}

\begingroup
\color{revcolor}
\subsection{Late-time exponents of the adaptive optimizers}
\label{app:late-time}

The eight-batch sweeps of Appendix~\ref{app:experiments} run for $10^{4}$ steps, and the exponents of Table~\ref{tab:corrected-sD} are fitted over $t \geq 10^{3}$. To test whether they are asymptotic, we repeated the sweeps for the PowerAdam family, Adam, Muon and SignGD with runs of $10^{5}$ steps (constant rate, standard settings, $\beta_2 = 0.999$, two teachers, with teachers and data streams shared across learning rates as in Appendix~\ref{app:common-teacher}), and fitted $\expp = 1 - s_B$, $\expq = -3 s_D$, the trajectory exponents and the joint envelope in successive half-decade windows of training steps.

For the PowerAdam family and Adam, $\expp$ and $\expq$ approach the nominal power $\pwr$ of an idealized $A = V^{-\pwr}$ preconditioner late in training (Table~\ref{tab:late-time}, Figure~\ref{fig:fs-late}): over $t \geq 3\times10^{4}$, $\expp$ lies within $0.01$ of $\pwr$ and $\expq$ within $0.02$, except for PowerAdam($1/4$), whose $\expq$ is still $0.06$ above it at the end of the runs. The envelope stays at $\data^{-1/3}$ and $2\ar + \at$ within $0.015$ of one, and the trajectory exponents approach their decoupled values for $\expq = \pwr$, $\ar = 1/(3-\pwr)$ and $\at = (1-\pwr)/(3-\pwr)$, with $\ar$ a few percent higher through the finite-norm slope of Remark~\ref{rem:adam-edge}. The excess of $\expq$ in the earlier window (Appendix~\ref{app:corrected}) is therefore pre-asymptotic.

How long the excess lasts depends on the memory of the second moment. With $\beta_2 = 0.9999$ instead of $0.999$, Adam's $\expq$ over the same late window is $0.66$ rather than $0.52$, consistent with a second moment that lags behind the current gradient statistics while its memory, $(1-\beta_2)^{-1}$ steps, is not short compared with the elapsed training; the excess is not, however, a function of that ratio alone, and an early-time contribution fades independently of the memory. Shorter memories do not reach the nominal values: with $\beta_2 = 0.9$ or $0.5$, Adam leaves the $1/3$ regime, as the second moment then follows individual bursts of the per-coordinate gradient and the update becomes sign-like. Muon's exponents change little at late times ($\expq$ from $0.64$ to $0.60$, $\expp \approx 0.43$), so its deviation has a different origin; SignGD's envelope leaves the $1/3$ regime late in training even across batch sizes, so it has no late-time exponents.

\begin{table}[h]
\centering
\begin{tabular}{lccccccc}
\toprule
optimizer & $\pwr$ & $\expq$ & $\expp$ & $\aD$ & $\ar$ & $\at$ & $2\ar + \at$ \\
\midrule
PowerAdam($1/8$) & $0.125$ & $0.141$ & $0.129$ & $0.334$ & $0.358$ & $0.280$ & $0.996$ \\
PowerAdam($1/4$) & $0.25$ & $0.308$ & $0.245$ & $0.340$ & $0.377$ & $0.238$ & $0.991$ \\
PowerAdam($3/8$) & $0.375$ & $0.392$ & $0.383$ & $0.335$ & $0.397$ & $0.192$ & $0.986$ \\
Adam & $0.5$ & $0.517$ & $0.507$ & $0.336$ & $0.415$ & $0.154$ & $0.985$ \\
Muon & -- & $0.597$ & $0.434$ & $0.330$ & $0.421$ & $0.118$ & $0.960$ \\
\bottomrule
\end{tabular}
\caption{\textbf{Late-time exponents.} Fits over $t \geq 3\times10^{4}$ of $10^{5}$-step runs at eight batch sizes ($\beta_2 = 0.999$); $\pwr$ is the nominal power of the preconditioner, $\aD$ the exponent of the joint envelope, and $\ar$, $\at$ the trajectory exponents of the runs at and next to each batch size's optimum.}
\label{tab:late-time}
\end{table}

\begin{figure}[h]
\centering
\includegraphics[width=0.55\linewidth]{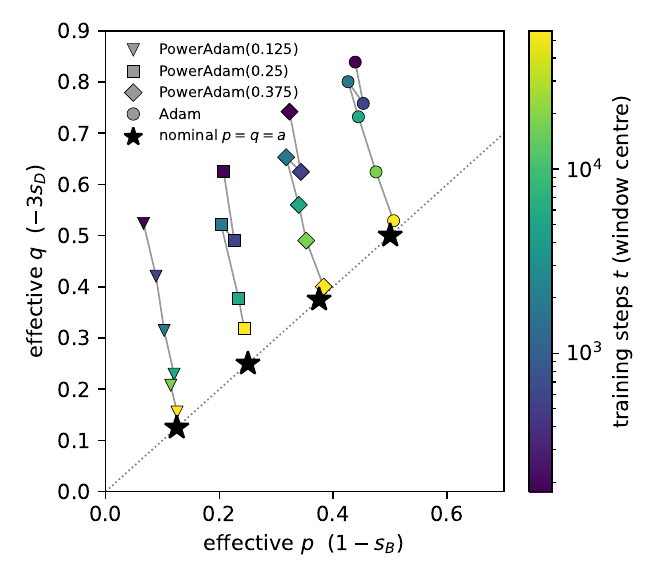}
\caption{\textbf{Late in training, the effective exponents approach their nominal values.} Effective $\expp = 1 - s_B$ and $\expq = -3 s_D$ for PowerAdam($\pwr$) and Adam ($\pwr = 1/2$), fitted in successive half-decade windows of training steps (color) over eight batch sizes at $\beta_2 = 0.999$; only windows centered after $100$ steps in which the optimal envelope scales as $\data^{-1/3}$ are shown. Each optimizer moves from an early excess in $\expq$ to its nominal $\expp = \expq = \pwr$ (stars) on the line $\expp = \expq$ (dotted). The exponents of Table~\ref{tab:corrected-sD}, fitted over $t = 10^{3}$--$10^{4}$, correspond to the third and fourth points of each path.}
\label{fig:fs-late}
\end{figure}
\endgroup

\begingroup
\color{revcolor}

\section{Learning-rate schedules}
\label{app:schedules}

\subsection{Setup}
\label{app:sched-setup}

These runs use the common-teacher procedure of Appendix~\ref{app:common-teacher}.

SGD ($\bs = 256$) and Adam ($\bs = 64$) are trained at nine budgets from $625$ to
$160{,}000$ steps, with the peak learning rate swept at five values per decade
and tuned separately at every budget. Every schedule starts with a linear warmup
over the first $2\%$ of the run; with $s$ the fraction of the run after warmup,
the rate is the peak times $1$ (constant), $\zeta + (1-\zeta)(1 + \cos\pi s)/2$
(cosine to a floor $\zeta \in \{20\%, 10\%, 1\%, 0\}$), or, for warmup--stable--decay
(WSD), $1$ followed by a linear decay to zero over the last $20\%$ of the run; a DeepSeek-like WSD \citep{deepseek2024v3} holds the peak for $70\%$ of the run, decays by a cosine to
$10\%$ of the peak over the next $20\%$, and holds that rate to the end. For SGD,
the constant schedule and the $20\%$ and $10\%$ floors are extended to $1.28$M
steps ($\data \approx 3.3 \times 10^{8}$).

\paragraph{Bounding the peak rate.} In a deep network the peak learning rate is
bounded by stability; our single-layer model has no such limit at the rates we
sweep (no run diverges), so we impose one. At each budget, the peak rate is
allowed up to $k$ steps of the learning-rate grid ($k$ factors of $10^{0.2} \approx 1.58$) above the constant-rate optimum at that budget, and the envelope is the minimum final loss over allowed runs, so that every envelope point is a trained run, as in Figure~\ref{fig:fig3}c. For SGD the constant-rate optimum does not depend on the budget, so the bound is a fixed learning rate; for adaptive optimizers it moves with the budget as the constant-rate optimum does, which fixes it in the scaling variable $\xc\data^{\expq/3}$ of Appendix~\ref{app:corrected}. Figure~\ref{fig:fig5} uses $k = 2$, a factor of $2.5$.

\subsection{Results}

\begin{table}[h]
\centering
\resizebox{\linewidth}{!}{\begin{tabular}{lcccccc}
\toprule
& \multicolumn{3}{c}{SGD} & \multicolumn{3}{c}{Adam} \\
\cmidrule(lr){2-4}\cmidrule(lr){5-7}
& \multicolumn{2}{c}{bound $2.5\times$} & unbounded & \multicolumn{2}{c}{bound $2.5\times$} & unbounded \\
schedule & $\aD$ & $\Ltot^{\ast}/\Ltot^{\ast}_{\mathrm{const}}$ & $\aD$ & $\aD$ & $\Ltot^{\ast}/\Ltot^{\ast}_{\mathrm{const}}$ & $\aD$ \\
\midrule
constant & $0.331$ & $1$ & $0.331$ & $0.332$ & $1$ & $0.332$ \\
cosine to $20\%$ & $0.340$ & $0.59$ & $0.357$ & $0.332$ & $0.71$ & $0.337$ \\
cosine to $10\%$ & $0.340$ & $0.56$ & $0.365$ & $0.334$ & $0.67$ & $0.322$ \\
cosine to $1\%$ & $0.343$ & $0.54$ & $0.748$ & $0.342$ & $0.65$ & $0.566$ \\
cosine to $0$ & $0.345$ & $0.54$ & $0.768$ & $0.344$ & $0.65$ & $0.705$ \\
WSD & $0.360$ & $0.41$ & $0.683$ & $0.359$ & $0.51$ & $0.580$ \\
cubic to $0$ & $0.342$ & $0.73$ & $0.842$ & $0.349$ & $0.87$ & $0.750$ \\
DeepSeek-like WSD & $0.341$ & $0.44$ & $0.377$ & $0.325$ & $0.55$ & $0.303$ \\
\bottomrule
\end{tabular}}
\caption{Envelope exponents over the last four budgets, and loss relative to a
constant rate at the largest budget, with the peak rate bounded by two grid steps ($2.5\times$) above the constant-rate optimum and without a bound. Budgets reach $\data = 4\times10^{7}$
(SGD) and $10^{7}$ (Adam).}
\label{tab:schedules}
\end{table}

With the peak rate bounded, every schedule keeps $\aD$ within $0.03$ of $1/3$ for both optimizers, and the schedule changes only the coefficient
(Table~\ref{tab:schedules}). Whether it improves the coefficient depends on how
the schedule spends the run: with the peak bounded, norm can be built only by
time spent near the bound, so schedules that hold the peak and cool at the end
(WSD) gain most, $50$--$60\%$, while a cubic decay, which spends most of the run far below the bound, gains $10$--$30\%$, and at a bound equal to the constant-rate optimum is worse than a constant rate. Figure~\ref{fig:fs-caps} shows the same comparison for bounds of $0$ to $4$ grid steps, factors of $1$, $1.6$, $2.5$, $4.0$ and $6.3$.
A DeepSeek-like WSD, which holds the peak, decays to $10\%$ and holds that rate, also keeps $\aD \approx 1/3$, $0.333 \pm 0.012$ for SGD over the extended runs up to $\data \approx 3\times10^{8}$, and at the $2.5\times$ bound lowers the loss more than cosine decay to $10\%$ ($0.44$ and $0.55$ of the constant-rate loss for SGD and Adam, against $0.56$ and $0.67$).

Without a bound, schedules with a final rate that is a fixed fraction of the peak
still keep $\aD = 1/3$ once the budget is large enough for that floor: over the
extended SGD runs, $\aD = 0.334 \pm 0.006$, $0.335 \pm 0.015$ and
$0.330 \pm 0.011$ for the constant, $20\%$ and $10\%$ schedules, with the optimal
peak independent of the budget ($s_{D} = -0.004$, $-0.008$ and $0.004$), and the
$10\%$ floor lowering the loss about $6\times$ for SGD and $2.4\times$ for Adam.
Lower floors reach $1/3$ only after a longer crossover, which for a $1\%$ floor
lies beyond our budgets. Schedules that anneal to zero instead steepen the
envelope over our budgets (Table~\ref{tab:schedules}, unbounded columns), because
their optimal peak rises with the budget, for SGD from $100$ to $1000$ for
cosine to zero and from $160$ to $2500$ for cubic decay, against a constant-rate
optimum of $16$. Under any fixed bound the rise stops once the optimal peak
reaches the bound, and the envelope returns to $1/3$ at larger budgets, later for a higher bound (Figure~\ref{fig:fs-caps}): within our budgets the return is complete at $2.5\times$, while at $4.0\times$ and $6.3\times$ the cosine and cubic schedules lie within $0.03$ of $1/3$ over the last four budgets and WSD, the slowest, is still approaching it ($0.38$--$0.44$); we do not analyze the unbounded
regime further here.

\begin{figure}[h]
\centering
\includegraphics[width=\linewidth]{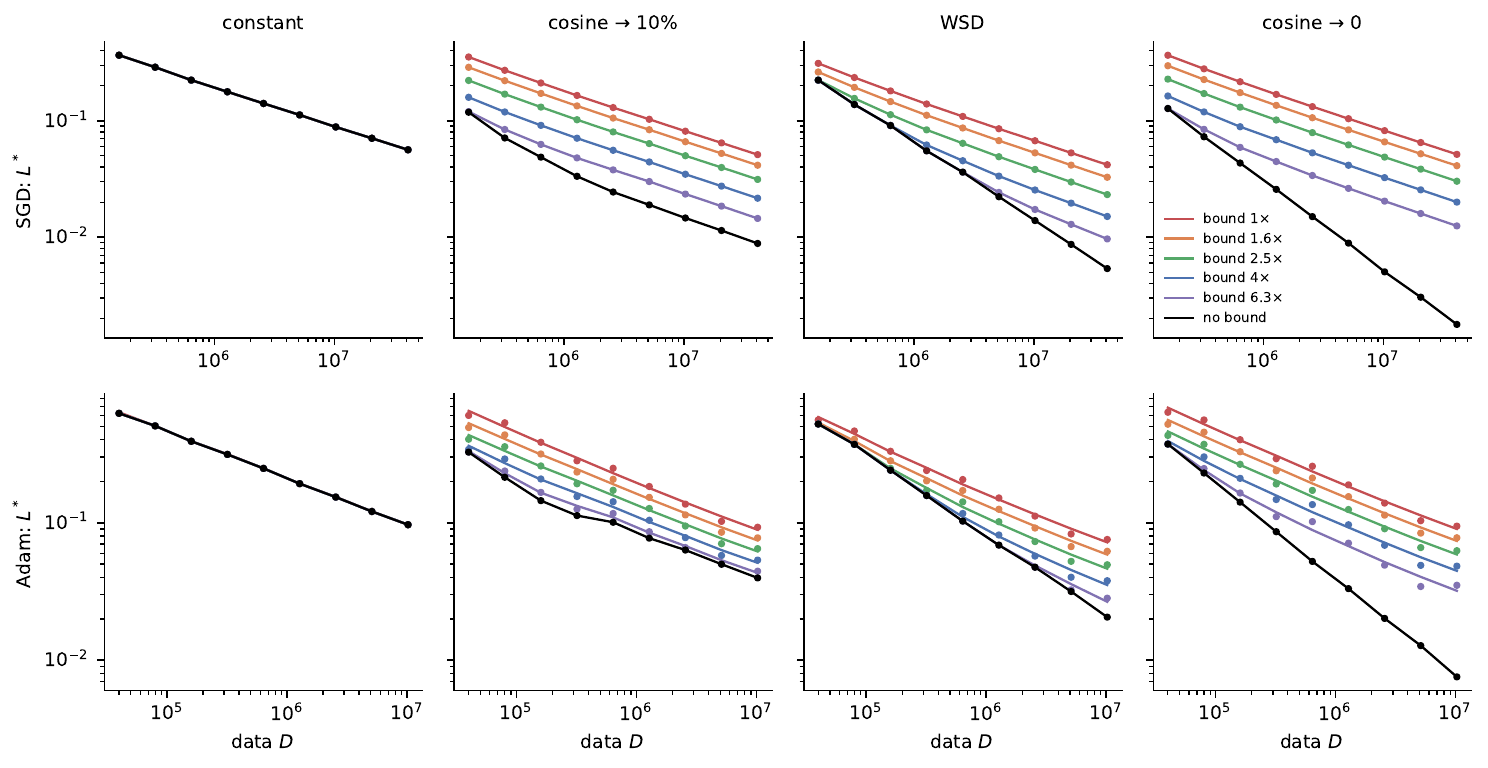}
\caption{\textbf{Envelopes under different bounds on the peak rate.} Optimal loss
for SGD (top) and Adam (bottom), with the peak rate bounded by $1$, $1.6$, $2.5$, $4.0$ and $6.3$ times the constant-rate optimum at each budget, and without a bound (black). Panels in a row share the vertical axis, so the constant-rate envelope (left), which scales as $\data^{-1/3}$, serves as the reference for the others. As in Figure~\ref{fig:fig5}, points are the final losses of the best trained runs with peak rate at most $0$ to $4$ grid steps above the constant-rate optimum, and lines are the same envelopes with the bound at exactly these factors of the constant-rate optimum, fitted across budgets, and the loss at the bound interpolated between grid rates. For SGD the two coincide, since its bounds fall on grid rates; for Adam the points scatter around the lines because its constant-rate optimum, and with it the grid bound, moves along the grid in discrete steps. With any finite bound the envelopes become parallel to the constant-rate one, $\aD = 1/3$, those of schedules that anneal to zero only after a steeper stretch at small budgets, before their optimal peak reaches the bound; without a bound they stay steeper.}
\label{fig:fs-caps}
\end{figure}

\subsection{Why a non-vanishing final rate or a bounded peak preserves $1/3$}

In the variables of Appendix~\ref{app:corrected}, write the tangential loss as
$\Lt = \Gamma/\bS$. With $\tau$ the cumulative learning rate, the SGD dynamics of Appendices~\ref{app:radial} and~\ref{app:corrected} read $d\bS^{3}/d\tau = 3(C + \Gamma)$, with $C/\bS = \Lr$, and
$d\Gamma/d\tau \simeq (\Gamma_{\mathrm{eq}}(\xc) - \Gamma)/\bS$: the tangential
distribution relaxes, over a cumulative learning rate of order $\bS$, toward an
equilibrium set by the current rate alone. For a schedule
$\xc(s) = \hat{\xc}f(s)$, $\tau$ is proportional to $\data$ at fixed shape and
peak. If the final rate is a fixed fraction $\zeta > 0$ of the peak, the last
stretch of the run at near-floor rates supplies a cumulative learning rate
proportional to $\data$, while the relaxation requirement $\bS$ grows only as
$\data^{1/3}$; for large enough $\data$ the tangential distribution is therefore
in equilibrium at the end of the run, with $\Gamma_{\mathrm{end}} =
\Gamma_{\mathrm{eq}}(\hat{\xc}\zeta)$. Then
$\bS^{3} = 3\data\,\mathcal{F}(\hat{\xc})$ for a function $\mathcal{F}$ of the peak
and the schedule shape, and
$\Ltot^{\ast} = \min_{\hat{\xc}}\,(C + \Gamma_{\mathrm{eq}}(\hat{\xc}\zeta))\,
\bS^{-1} \propto \data^{-1/3}$ with a $\data$-independent optimal peak. The budget
at which the end of the run reaches equilibrium grows as the floor decreases,
which is why lower floors reach $1/3$ later, and the argument does not apply
when the rate anneals to zero, where the relaxation time diverges at the end of
the run.

The same variables show why a bound on the peak restores $1/3$ even when the rate
anneals to zero. With the peak fixed at the bound, $\bS^{3}$ again grows in
proportion to $\data$, while extra budget only lengthens the decay, lowering the
heat that freezes in at the end. That heat cannot fall below zero, so once it is
small compared with $C$ the loss approaches $C/\bS \propto \data^{-1/3}$, from
above, with a coefficient set by the bound and the schedule.

\section{Anisotropic inputs}
\label{app:aniso}

We replace the isotropic inputs by $x = \mathrm{rms}(z M)$, $z \sim \mathcal{N}(0, I)$,
with a power-law covariance spectrum $\lambda_{k} \propto k^{-s}$ ($k = 1, \dots, 32$,
rescaled to mean one), either along the coordinate axes ($M = \mathrm{diag}\sqrt{\lambda}$)
or in a random orthonormal basis ($M = \mathrm{diag}\sqrt{\lambda}\,Q^{\top}$). The
participation ratio falls from $32$ to about $10$ for $s = 1$ and $2.4$ for
$s = 2$. Because the teacher is not adapted to the spectrum, strong anisotropy
also makes the classes unequal: the ten most frequent of the $128$ classes cover
$17\%$ of inputs for isotropic inputs, $23$--$34\%$ for $s = 1$ and $40$--$57\%$
for $s = 2$. SGD is rotation-equivariant, so it is run only along the axes;
Adam, whose preconditioning is per coordinate, is run in both bases. Runs use a
constant rate without warmup and otherwise follow
Appendix~\ref{app:sched-setup}.

\begin{table}[h]
\centering
\resizebox{\linewidth}{!}{\begin{tabular}{llccc}
\toprule
optimizer & inputs & $\aD$ (all budgets) & $\aD$ (final decade) & $s_{D}$ \\
\midrule
SGD & isotropic & $0.337 \pm 0.001$ & $0.331 \pm 0.007$ & $-0.014 \pm 0.005$ \\
SGD & $s = 1$, axes & $0.344 \pm 0.003$ & $0.351 \pm 0.013$ & $-0.002 \pm 0.012$ \\
SGD & $s = 2$, axes & $0.414 \pm 0.011$ & $0.340 \pm 0.013$ & $0.212 \pm 0.008$ \\
\midrule
Adam & isotropic & $0.339 \pm 0.003$ & $0.332 \pm 0.015$ & $-0.227 \pm 0.007$ \\
Adam & $s = 1$, axes & $0.340 \pm 0.004$ & $0.341 \pm 0.005$ & $-0.221 \pm 0.006$ \\
Adam & $s = 1$, rotated & $0.340 \pm 0.003$ & $0.336 \pm 0.015$ & $-0.234 \pm 0.004$ \\
Adam & $s = 2$, axes & $0.344 \pm 0.005$ & $0.339 \pm 0.008$ & $-0.222 \pm 0.006$ \\
Adam & $s = 2$, rotated & $0.354 \pm 0.005$ & $0.350 \pm 0.015$ & $-0.211 \pm 0.010$ \\
\midrule
\multicolumn{5}{l}{extended budgets, $\data = 10^{7}$ to $8\times10^{7}$:} \\
Adam & isotropic & \multicolumn{2}{c}{$0.326 \pm 0.009$} & $-0.13 \pm 0.03$ \\
Adam & $s = 2$, rotated & \multicolumn{2}{c}{$0.323 \pm 0.014$} & $-0.21 \pm 0.02$ \\
\bottomrule
\end{tabular}}
\caption{Envelope and optimal-learning-rate exponents for anisotropic inputs.
Only exponents are compared across rows: each spectrum defines a different task,
so the losses themselves are not comparable. The extended runs cover less than a
decade, so their $s_{D}$ values are less precise.}
\label{tab:aniso}
\end{table}

Mild anisotropy leaves $\aD$ and $s_{D}$ unchanged for both optimizers. Strong
anisotropy slows the approach to the asymptotic regime: for SGD the exponent
fitted over all budgets is $0.41$, but $0.34$ over the final decade, with the
optimal rate still rising across the range; for Adam with a rotated spectrum the
exponent is $0.35$ up to $\data = 10^{7}$, and extending the runs to
$\data = 8\times10^{7}$ gives $\aD = 0.323 \pm 0.014$, matching the isotropic
$0.326 \pm 0.009$ over the same range. Strong anisotropy therefore delays the
approach to $1/3$ for both optimizers without changing the exponent. A plausible
reading is that low-variance directions enter the boundary layer only once $\bS$
is large enough to resolve them, which lengthens the transient; we have not
tested it.

\endgroup

\end{document}